\documentclass[11pt]{article}
\usepackage{amssymb,amsmath,bm,graphicx}
\usepackage{amsthm}
\usepackage{mathrsfs}
\usepackage{constants}
\usepackage{tikz}
\usepackage{hyperref}
\usepackage{enumitem}
\usepackage{textcomp}
\usepackage{url}
\usepackage{verbatim}
\usepackage{cite}
\usepackage{multirow}
\usepackage{multicol}
\usepackage{xcolor}
\usepackage{microtype}
\usepackage{longtable}
\usepackage{tabto}
\usepackage{tabularx}
\usepackage{amsmath}
\usepackage{float}
\usepackage{subcaption}
\usepackage{caption}
\usepackage{algpseudocode}
\usepackage{algorithm}
\usepackage{color}

\newcommand{\rcolor}[1]{\textcolor{red}{#1}}

\usepackage{float}

\usepackage{placeins}

\usepackage{amssymb}
\usepackage{graphicx}

\usepackage{subcaption}
\usepackage{caption}

\usepackage{csvsimple}
\usepackage[T2A]{fontenc}
\usepackage[utf8]{inputenc}
\usepackage{babel}
\usepackage{dblfloatfix}
\usepackage{amsmath}
\usepackage{lipsum}
\usepackage{amsfonts}
\usepackage{caption}

\usepackage{newtxmath,booktabs,sectsty}
\paragraphfont{\mdseries\itshape}
\usepackage{tabularx,ragged2e}
\newcolumntype{L}[1]{>{\RaggedRight\hsize=#1\hsize}X}
\newcolumntype{C}[1]{>{\Centering\hsize=#1\hsize\hspace{0pt}}X}

\date{}
\begin{document}
\newcommand{\bea}{\begin{eqnarray}}
\newcommand{\ena}{\end{eqnarray}}
\newcommand{\beas}{\begin{eqnarray*}}
\newcommand{\enas}{\end{eqnarray*}}
\newcommand{\beq}{\begin{equation}}
\newcommand{\enq}{\end{equation}}
\def\qed{\hfill \mbox{\rule{0.5em}{0.5em}}}
\newcommand{\bbox}{\hfill $\Box$}
\newcommand{\ignore}[1]{}
\newcommand{\ignorex}[1]{#1}
\newcommand{\wtilde}[1]{\widetilde{#1}}
\newcommand{\qmq}[1]{\quad\mbox{#1}\quad}
\newcommand{\qm}[1]{\quad\mbox{#1}}
\newcommand{\nn}{\nonumber}
\newcommand{\Bvert}{\left\vert\vphantom{\frac{1}{1}}\right.}
\newcommand{\To}{\rightarrow}
\newcommand{\E}{\mathbb{E}}
\newcommand{\Var}{\mathrm{Var}}
\newcommand{\Cov}{\mathrm{Cov}}
\newcommand{\Corr}{\mathrm{Corr}}
\newcommand{\dist}{\mathrm{dist}}
\newcommand{\diam}{\mathrm{diam}}
\makeatletter
\newsavebox\myboxA
\newsavebox\myboxB
\newlength\mylenA
\newcommand*\xoverline[2][0.70]{%
    \sbox{\myboxA}{$\m@th#2$}%
    \setbox\myboxB\null
    \ht\myboxB=\ht\myboxA%
    \dp\myboxB=\dp\myboxA%
    \wd\myboxB=#1\wd\myboxA
    \sbox\myboxB{$\m@th\overline{\copy\myboxB}$}
    \setlength\mylenA{\the\wd\myboxA}
    \addtolength\mylenA{-\the\wd\myboxB}%
    \ifdim\wd\myboxB<\wd\myboxA%
       \rlap{\hskip 0.5\mylenA\usebox\myboxB}{\usebox\myboxA}%
    \else
        \hskip -0.5\mylenA\rlap{\usebox\myboxA}{\hskip 0.5\mylenA\usebox\myboxB}%
    \fi}
\makeatother

\newtheorem{theorem}{Theorem}[section]
\newtheorem{corollary}[theorem]{Corollary}
\newtheorem{conjecture}[theorem]{Conjecture}
\newtheorem{proposition}[theorem]{Proposition}
\newtheorem{lemma}[theorem]{Lemma}
\newtheorem{definition}[theorem]{Definition}
\newtheorem{example}[theorem]{Example}
\newtheorem{remark}[theorem]{Remark}
\newtheorem{case}{Case}[section]
\newtheorem{condition}{Condition}[section]

\title{{\bf\Large Ranked differences Pearson correlation dissimilarity with an application to electricity users time series clustering}}
\author{Chutiphan Charoensuk and Nathakhun Wiroonsri  \\ Statistics, Probability, and Data Science with R programming (SPDR) research group \\ Department of Mathematics, King Mongkut's University of Technology Thonburi}


\maketitle

\begin{abstract}

Time series clustering is an unsupervised learning method for classifying time series data into groups with similar behavior. It is used in applications such as healthcare, finance, economics, energy, and climate science. Several time series clustering methods have been introduced and used for over four decades. Most of them focus on measuring either Euclidean distances or association dissimilarities between time series. In this work, we propose a new dissimilarity measure called ranked Pearson correlation dissimilarity (RDPC), which combines a weighted average of a specified fraction of the largest element-wise differences with the well-known Pearson correlation dissimilarity. It is incorporated into hierarchical clustering. The performance is evaluated and compared with existing clustering algorithms. The results show that the RDPC algorithm outperforms others in complicated cases involving different seasonal patterns, trends, and peaks. Finally, we demonstrate our method by clustering a random sample of customers from a Thai electricity consumption time series dataset into seven groups with unique characteristics.

\end{abstract}




\section{Introduction} \label{sec:introduction}

Cluster analysis is a fundamental process in statistics and machine learning in which data are grouped into subsets, known as clusters, according to their similarity or dissimilarity (refer to the book by \cite{statbook2021} for a review). The primary goal is to uncover hidden patterns or structures within the data without supplying predefined group labels; thus, it is a form of unsupervised learning. Cluster analysis is widely applied across fields including business, medicine, and the environmental and social sciences. For instance, it is used to identify groups of customers with similar purchasing behavior, group cancer cells by biochemical characteristics, and analyze wildfire risk areas using satellite data (see \cite{exbio,exfire,charoensuk2025rankeddifferencespearsoncorrelation} for examples). Some notable algorithms are K-means clustering, hierarchical clustering, fuzzy C-means clustering, EM clustering, and DBSCAN (see \cite{kmean,hierarchy,fuzzy,em,dbscan} for more details).

Time series clustering is a special case of cluster analysis in which input features are inherently dependent on time. Several specialized algorithms have been proposed for this task. Shape-based clustering using dynamic time warping (DTW) \cite{dtw} stands out for its ability to align sequences by dynamically warping the time axis to minimize the distance between them. This approach is particularly robust to time shifts and variations in speed, making it highly effective for time-dependent data. Correlation-based clustering (CBC) considers the correlation between variables instead of traditional distance measures, such as the Euclidean distance used in K-means or hierarchical clustering. CBC focuses on the pattern similarity of the data and uses a correlation measure such as the Pearson, Spearman, or Kendall correlation. The K-means clustering method assigns data to a fixed number of groups so that the members of each group are closer to its centroid than that of any other group. In contrast, hierarchical clustering builds a tree-like structure that represents the relationships between clusters, which are progressively merged or split to obtain an appropriate clustering. This can be visualized with a diagram called a dendrogram.

Time series clustering has been popular and has several applications lately (see for instance \cite{consumersegmentation,tpt,coclus}). Moreover, many time series clustering techniques have been widely applied in numerous fields. For example, correlation-based hierarchical clustering (see \cite{correlationbasedspatial}) focuses on capturing relationships between time series and spatial constraints, enhancing the accuracy and efficiency of data analysis. DTW has many applications (see \cite{derridtw,tskmean,kmeansusingdtw} for instances). One study \cite{solarkmeans} presented a method for probabilistic solar radiation intensity forecasting by using K-means time series clustering to group data according to solar radiation patterns. The experimental results showed that this method significantly improves forecasting accuracy relative to other techniques. Typically, one of these algorithms is applied in isolation, focusing on its effectiveness in solving specific clustering problems.

Most studies have focused on either the Euclidean distance between two time series, such as Euclidean and DTW approaches, or the association between them through correlation coefficients. However, both perspectives are important in some applications. For instance, it is beneficial to classify electricity consumers on the basis of both monthly usage and behavioral associations. If, for instance, customers A and B have similar average electricity usage but substantially different seasonal patterns, then they should be classified into distinct groups. To illustrate this, Figure~\ref{threepea} shows two pairs of Provincial Electricity Authority (PEA) customers who use similar amounts of electricity overall but with different behavioral patterns. Electricity consumers have been clustered in previous studies, but only with fuzzy C-means and K-means clustering methods (see \cite{ElecclusterbyFCM,kmeansclusteringelectricconsumer}). In \cite{Pearsondis}, it was shown that the Euclidean distance, when applied to scaled data, is equivalent to the Pearson correlation dissimilarity. However, for electricity consumption data, scaling is not appropriate as it affects their background meaning.

\begin{figure}[H]
\centering\includegraphics[width=12cm]{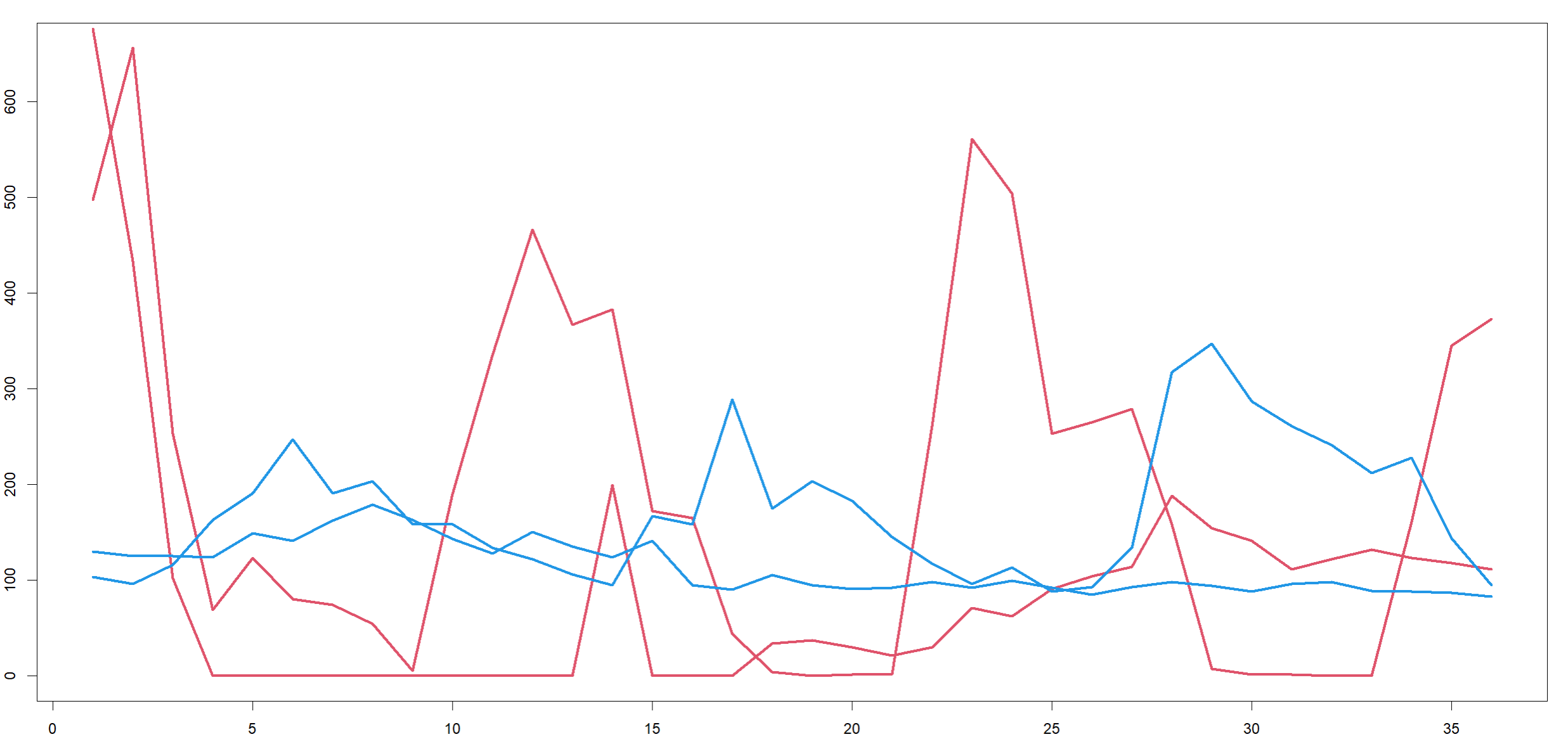}
\caption{Examples of PEA users' behaviors}
\label{threepea}
\end{figure}

This limitation motivates us to introduce a new dissimilarity measure that combines a simple absolute difference and a correlation coefficient, called the ranked-differences Pearson correlation dissimilarity (RDPC). This is defined as an interpolation between the Pearson correlation and a weighted average of a specified fraction of the largest element-wise differences between two time series. The RDPC is then attached to the hierarchical clustering algorithm with different linkages. We test the performance of our dissimilarity on four groups of artificial datasets. The last two groups are intended to imitate electricity consumers' behaviors, as shown in Figure~\ref{threepea} and discussed previously. 

The remainder of this work is organized as follows. Section~\ref{sec:background} provides background information regarding clustering algorithms and dissimilarities. Our proposed dissimilarity measure is defined and its mathematical properties discussed in Section~\ref{sec:main}. Experimental results on simulated datasets are presented in Section~\ref{sec:exp}. Section~\ref{sec:app} considers an application to PEA electricity consumers. The concluding remarks and future research directions are discussed in Section~\ref{sec:conclusion}.

\section{Background} \label{sec:background}

In this section, we review the existing definitions, terms, and methods used in this work, including clustering algorithms, determining the number of clusters, dissimilarity measures, and distances. Most of the clustering algorithms are defined in general terms and are compatible with several distance or dissimilarity measures.

\subsection{Clustering algorithm}
In this subsection, we recall the definitions of hierarchical clustering and K-means clustering.
Let $n,k,p \in \mathbb{N}$, $[n] = \{1,2,\ldots,n\}$, $i\in[n]$, and $j \in [k]$. We establish the following notation used in this work: 
\begin{enumerate} \label{notations}
    \item $X = (X_{1},X_{2},\ldots,X_{n})$: Random data points.
    \item $x = (x_{1},x_{2},\ldots,x_{n})$: Data points.
    \item K: The actual number of clusters.
    \item $C_j$: The set of data points in the $j$th cluster.
    \item $v_j$: The $j$th cluster centroid.
    \item $v_0$: The centroid of the entire dataset.
    \item $\bar{v}$: The centroid of all the $v_j$
    \item $\| x-y \|$: The Euclidean distance between $x$ and $y$.
    \item $\Corr(\cdot,\cdot)$: The correlation coefficient. In this work, we consider only the Pearson correlation.
    \item $e_{i}$: The $i$th elbow point.
\end{enumerate}

\subsubsection{Hierarchical clustering} 
Hierarchical clustering \cite{hierarchy} seeks to build a hierarchy of clusters, often visualized in a tree-like structure called a dendrogram, as shown in Figure~\ref{dendrogram}. It begins by treating each data point as a cluster and iteratively merging the nearest clusters to form a hierarchy using a distance or a dissimilarity measure. The dissimilarity between two clusters is defined using a linkage such as the complete, average, or single linkage. In this work, we use the complete linkage, defined as the maximal intercluster dissimilarity, by computing all pairwise dissimilarities between the observations in two clusters and recording the largest one. The final number of clusters is determined at the end by cutting the dendrogram at a specified height. 

\begin{figure}[H]
\centering\includegraphics[width=16cm]{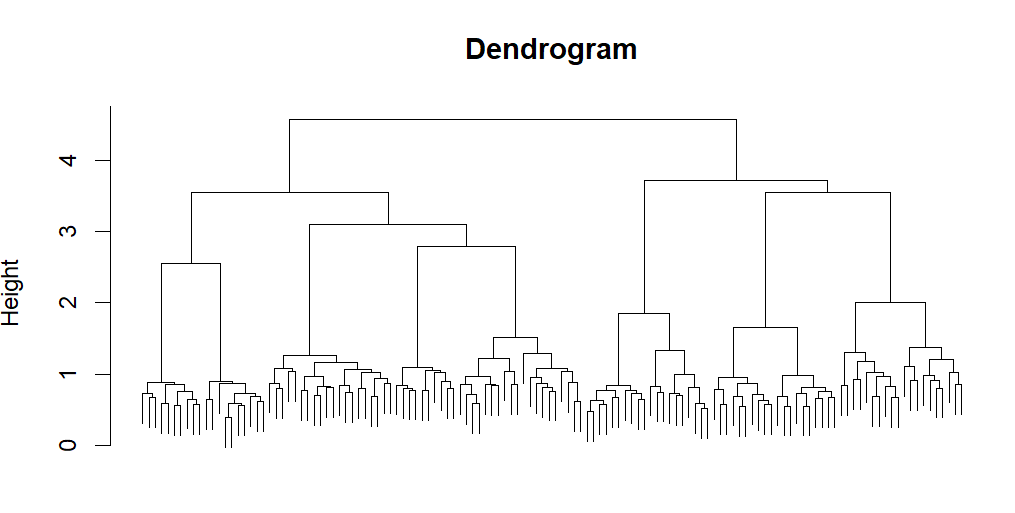}
\caption{Dendrogram}
\label{dendrogram}
\end{figure}

\subsubsection{K-means}
The K-means method \cite{kmean} is a simple yet efficient clustering algorithm. It operates by partitioning a dataset into $k$ clusters, where $k$ is a user-defined parameter. The algorithm commences by randomly initializing cluster centroids. It then assigns each data point to the nearest centroid and updates the centroids based on the newly assigned points. This iterative process continues until the cluster centroids converge. The objective of K-means clustering is to minimize the within-cluster sum of squares, expressed as

\begin{equation*}
    \sum_{j=1}^k\sum_{x \in C_j} \|x-v_j\|^2 .
\end{equation*}

\subsection{Elbow method}
Elbow method \cite{elbow} is a technique used in data analysis and machine learning to identify the optimal number of clusters within a dataset. The method works by plotting the Within-Cluster dissimilarity against the number of clusters (K) and observing the point where the rate of decrease sharply levels off, creating a distinctive elbow shape. This point indicates the candidate number of clusters to use for the data. An example is shown in Figure \ref{background elbow}.

\begin{figure}[H]
\centering\includegraphics[width=12cm]{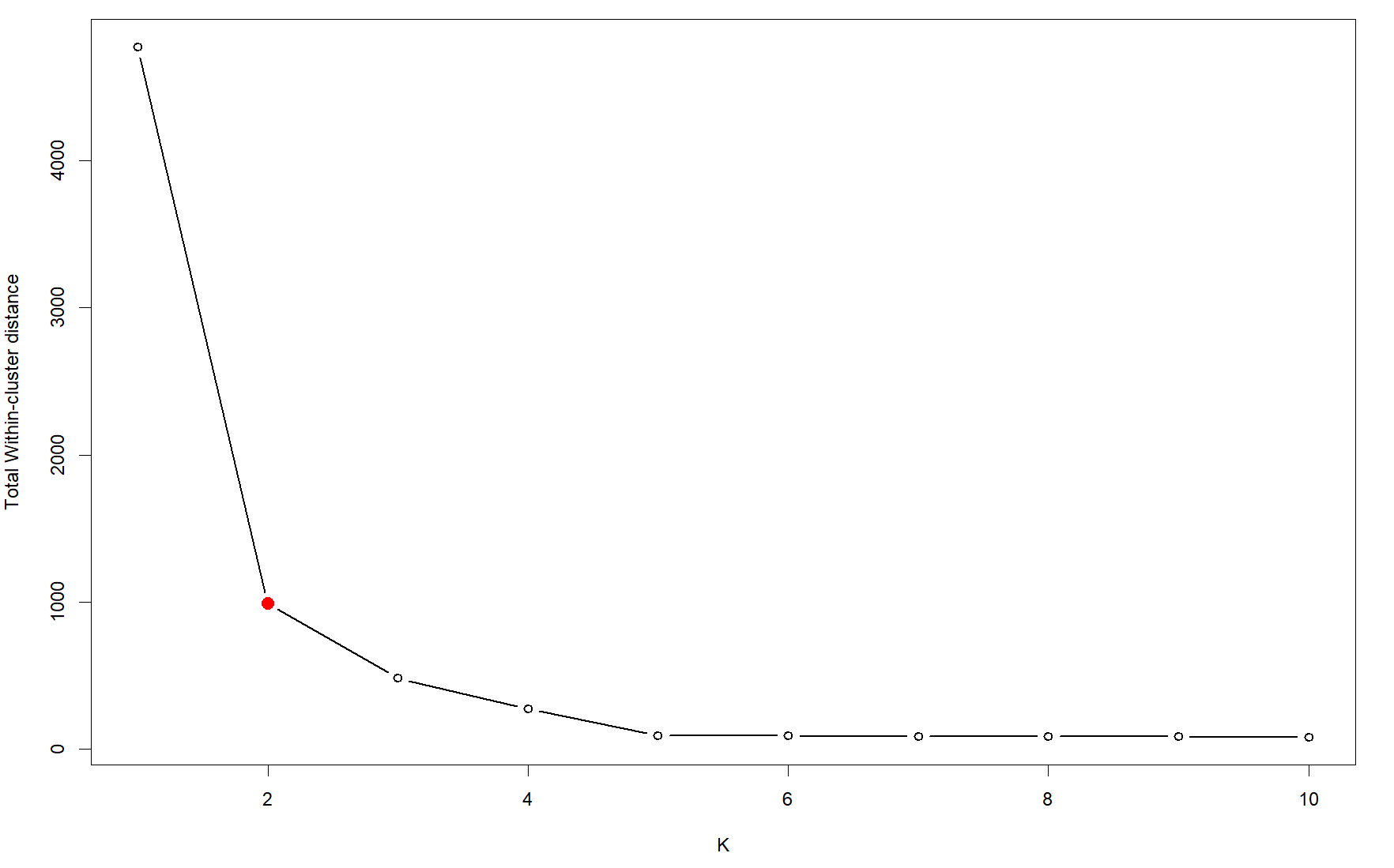}
\caption{Elbow point}
\label{background elbow}
\end{figure}

\subsection{Dissimilarity measures}
In this subsection, we introduce the distance and dissimilarity measures used in this work.

\subsubsection{Dynamic time warping}

DTW \cite{dtw} is a powerful algorithm used in time series analysis to measure the similarity between two temporal sequences. Unlike traditional distance metrics like the Euclidean distance, DTW can handle sequences of different lengths. As shown in Figure~\ref{warping}, when two sequences are not synchronized, it warps their time axes to align them in a way that minimizes the distance between corresponding points.

\begin{figure}[H]
\centering\includegraphics[width=8cm]{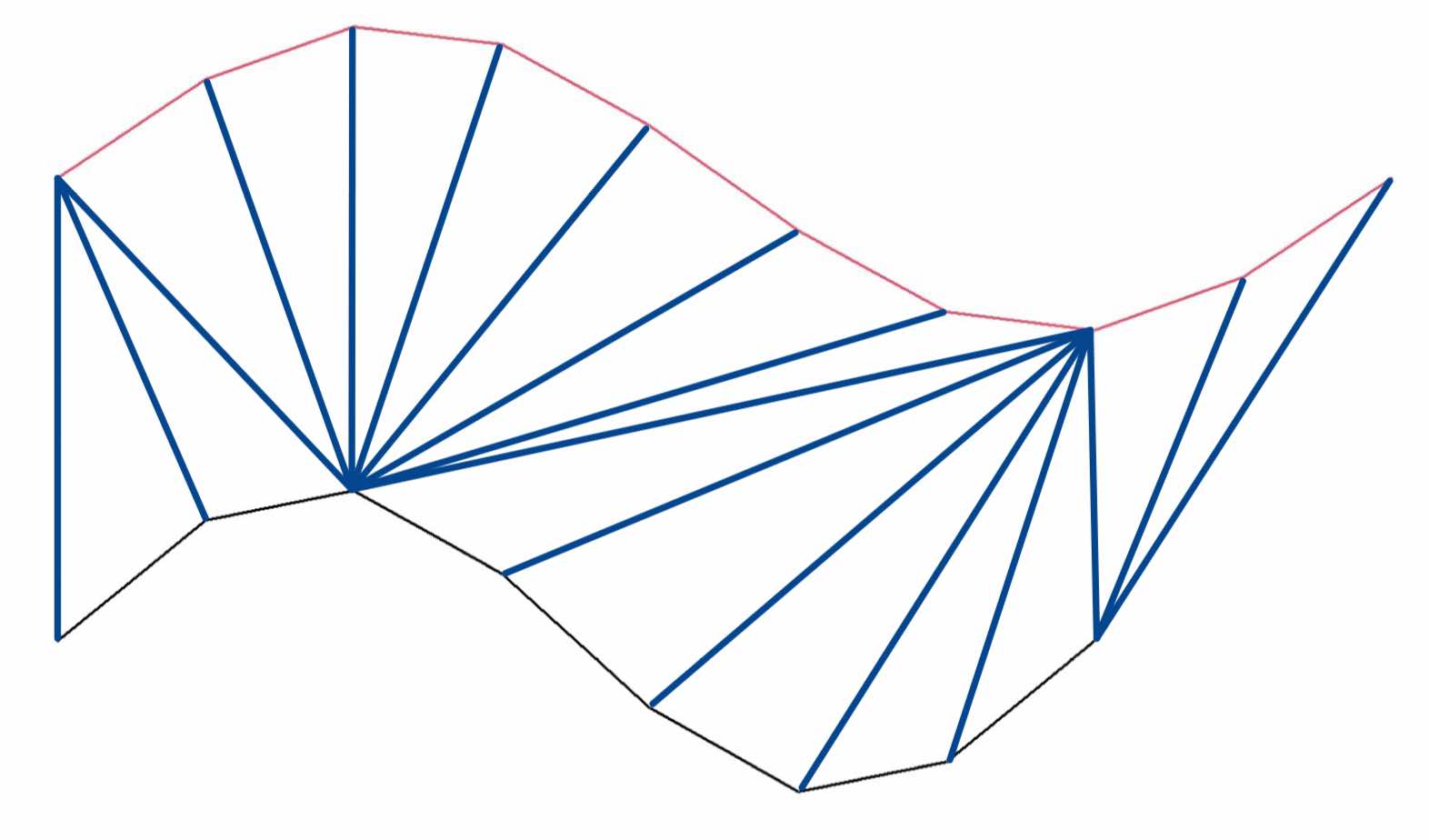}
\caption{DTW Alignment}
\label{warping}
\end{figure}

\subsubsection{Euclidean distance}

Let $x ,y \in \mathbb{R}^n$. The Euclidean distance between $x$ and $y$ is defined as  
\begin{equation}
||x-y|| = \sqrt{(x_{1}-y_{1})^2+(x_{2}-y_{2})^2+\ldots+(x_{n}-y_{n})^2}.
\end{equation}

\subsubsection{Global alignment kernel}
The global alignment kernel (GAK) \cite{GAK} is a kernel function used primarily for sequential data, such as time series or strings. It is based on DTW and measures the similarity between two sequences by aligning them in a way that minimizes the distance between their elements. The difference between GAK and DTW is that GAK is commonly used to compare sequences of similar length, considering the entire sequence. In contrast, DTW is preferred when sequences have different lengths or when events occur at uneven rates. GAK is particularly useful in applications like bioinformatics and speech recognition.

\subsubsection{Pearson correlation dissimilarity}
The Pearson correlation coefficient quantifies the degree to which a linear relationship exists between two continuous variables. The coefficient ranges from $-1$ to 1, where 1 indicates a perfect positive linear relationship,
$-1$ indicates a perfect negative linear relationship, and 0 indicates no linear relationship. The correlation between two vectors $x = (x_{1},x_{2},\ldots,x_{n})$ and $y= (y_{1},y_{2},\ldots,y_{n})$ is defined as

\begin{equation*}
\Corr(x,y) = \frac{\frac{1}{n}\sum_{i=1}^n(x_i-\bar{x})(y_i-\bar{y})}{s_{x}s_{y}},
\end{equation*}
where $\bar{x}$ and $\bar{y}$ denote the means of $x$ and $y$, respectively, and $s_{x}$ and $s_{y}$  denote their standard deviations.

The Pearson correlation dissimilarity \cite{Pearsondis} is defined as

\begin{equation}
d_{P}(x,y) = 1 - \Corr(x,y).
\end{equation}

In this work, we compare the combinations of clustering algorithms and dissimilarity measures shown in Table~\ref{5clmethods}, including the proposed RDPC measure.

\begin{table*}[h]
\centering
\resizebox{12cm}{!}{%
\begin{tabular}{|c|c|}
\hline
\multicolumn{2}{|c|}{\textbf{Clustering Method}} \\
\hline
\textbf{Clustering algorithm} & \textbf{Distance or dissimilarity measure} \\
\hline
 & Pearson correlation \\
Hierarchical clustering  & DTW\\
& GAK \\
& RDPC \\
\hline
K-means & Euclidean distance \\
\hline
\end{tabular}%
}
\caption{Combinations of clustering algorithms and dissimilarity measures}
\label{5clmethods}
\end{table*}

\section{Main results}\label{sec:main}
In this section, we define our new dissimilarity measure and discuss its mathematical properties.    

\subsection{Definition}

Our newly introduced dissimilarity is a combination of a weighted average of a specified fraction of the largest element-wise differences and the well-known Pearson correlation dissimilarity measure. It is defined in the following two definitions.

\begin{definition}
Let $x, y \in \mathbb{R}^n$ and $R_i = |x_i-y_i|$ for all $i$. Letting $0\le p \le 1$, $r = \lceil pn \rceil$, $\vec{w} = (w_{1},w_{2},\ldots,w_{r}) \in \mathbb{R}^n$ be such that $w_i > 0$, and $\sum_{i=1}^{r} w_i = 1$, RankDiff is defined as
\begin{equation*}
\text{RankDiff}(x,y|p,\vec{w}) = \sum_{j=1}^r w_jR_{(n-j+1)}
\end{equation*}
where $R_{(j)}$ denotes the $j$th order statistic.

\end{definition}

\begin{definition}\label{RDPCdef}
Let $x, y \in \mathbb{R}^n$ and $0 \leq\alpha\leq 1$. $d_{RDPC}$ is defined as  
\begin{equation*}
d_{RDPC}(x,y|\alpha, p,\vec{w}) = \alpha\text{RankDiff}(x,y|p,\vec{w}) + (1-\alpha)d_{P}(x,y).
\end{equation*}
\end{definition}

Next, we discuss some reasonable choices of $\vec{w}$ and the choice we take for the rest of the work.
\begin{remark}
It is reasonable to take a monotone sequence of $w_j$. For instance,
\beas
w_j = \frac{2j}{r(r+1)} \text{ \ \ or \ \ } w_j = \frac{2(n-j+1)}{r(r+1)}
\enas
which correspond to the cases in which we rely more on the small or the large differences, respectively. However, for simplicity, we choose the uniform weights $w_j = 1/r$ for the remainder of this work.
\end{remark}

\subsection{Mathematical properties} 

In this subsection, we state and prove some mathematical properties of the RDPC dissimilarity. We first recall the definition of a metric or distance on a metric space $M$.

\begin{definition}
    Let $M$ be a metric space. A metric $d$ is a function $d:M \times M \rightarrow \mathbb{R}$ that satisfies the following axioms:
\begin{enumerate}
    \item[(M1)] $d(x,y) \geq 0$
    \item[(M2)] $d(x,y) = 0 \iff x = y$
    \item[(M3)] $d(x,y) = d(y,x)$
    \item[(M4)] $d(x,y) = d(x,z) + d(z,y)$,
\end{enumerate}
for all $x, y, z \in M$.
\end{definition}

Next, we consider which of these axioms are satisfied by $d_{RDPC}$.

\begin{proposition}
   $d_{RDPC}$ as defined in Definition~\ref{RDPCdef} satisfies \textit{(M1)}.
\end{proposition}
 \begin{proof}
Since $0\leq \alpha \leq1$, $0\leq 1-\alpha \leq1$, $\text{RankDiff}(x,y|p,\vec{w}) \geq 0$, and $0\leq d_{P}(x,y) \leq 2$,  we have that $d_{RDPC}(x,y|\alpha, p,\vec{w}) = \alpha\text{RankDiff}(x,y|p,\vec{w}) + (1-\alpha)d_{P}(x,y) \geq0$.
 \end{proof}
 
\begin{proposition}
$d_{RDPC}$ as defined in Definition~\ref{RDPCdef} satisfies \textit{(M2)} only for $\alpha>0$.
\end{proposition}
 \begin{proof}
 \textbf{Case $\alpha = 0$:} Consider $x = (1,2)$ and $y = (2,4)$.  It can be seen that $x \ne y$ but $d_{RDPC}(x,y|0, p,\vec{w}) = 1-\Corr(x,y) = 1-1 = 0.$   

\textbf{Case $\alpha = 1$:} Since $w_j>0$, it can be seen that $d_{RDPC}(x,y|1, p,\vec{w}) = \sum_{j=1}^r w_jR_{(n-j+1)} = 0 $ if and only if $R_{(n)} = 0$. This implies that $R_i = |x_i-y_i| = 0$ for all $i = 1, 2,\ldots, n$, which yields $x = y$. 

\textbf{Case $0<\alpha<1$:} $d_{RDPC}(x,y|\alpha, p,\vec{w})  = 0$ only if $\sum_{j=1}^r w_jR_{(n-j+1)} = 0 $. This reduces to the previous case.

\end{proof}

\begin{proposition}
 $d_{RDPC}$ as defined in Definition~\ref{RDPCdef} satisfies \textit{(M3)}.
\end{proposition}
 \begin{proof}
    Since $R_i = |x_i-y_i| = |y_i-x_i|$ and $d_{P}(x,y) = 1-\Corr(x,y) = 1-\Corr(y,x) = d_{P}(y,x)$, $d_{RDPC}(x,y|\alpha, p,\vec{w}) = d_{RDPC}(y,x|\alpha, p,\vec{w})$.      
 \end{proof}

\begin{proposition}
$d_{RDPC}$ as defined in Definition~\ref{RDPCdef} satisfies \textit{(M4)} when $\alpha=1$ but not necessarily in other cases.
\end{proposition}

\begin{proof}
We first show that $d_{RDPC}$ satisfies \textit{(M4)} for  $\alpha = 1$. Let $n \in \mathbb{N}$, and let $x = (x_{1},x_{2},\ldots,x_{n})$, $y= (y_{1},y_{2},\ldots,y_{n})$, and $z= (z_{1},z_{2},\ldots,z_{n})$ be three vectors in $\mathbb{R}^n$. Then $R_i = |x_i-y_i| = |x_i-z_i+z_i-y_i| \leq |x_i-z_i| + |z_i-y_i| := R^{xz}_{i} + R^{zy}_{i}$. Letting $R_{(n-j+1)} = R_{i_j}$ for $j=1,2,\ldots,r$, we have

\beas
 d_{RDPC}(x,y|1, p,\vec{w}) &=& \sum_{j=1}^r w_jR_{(n-j+1)} = \sum_{j=1}^r w_jR_{i_j} \\
                            &\le& \sum_{j=1}^r w_jR^{xz}_{i_j} + \sum_{j=1}^r w_jR^{zy}_{i_j} \\
                            &\le& \sum_{j=1}^r w_jR^{xz}_{(n-j+1)} + \sum_{j=1}^r w_jR^{zy}_{(n-j+1)} \\
                            &=& d_{RDPC}(x,z|1, p,\vec{w}) + d_{RDPC}(z,y|1, p,\vec{w}).
\enas
This completes the proof for $\alpha=1$.

Next, we give counterexamples to show that $d_{RDPC}$ does not necessarily satisfy \textit{(M4)} for $0\leq\alpha<1$ and $p>0$.

Let $x_1 = (1 ,-1.3  ,-0.7)$, $y_1 =(-0.9,-0.3,-1)$, $z_1 =(-0.2,-0.3,-0.3)$, $x_2 = (-2.8  ,0.5,  1.4)$, $y_2 =(1.0  ,0.3 ,-1.2)$, $z_2 =(-0.9  ,0.8 ,-0.5)$ $x_3 = (-0.2 ,0.4  ,0.1)$, $y_3 =(-2.2,-0.5,-2.0)$, and $z_3 =(0.8,0.0,-1.8)$. These three examples violate (M4) in the cases $0<p\le 1/3$, $1/3 < p \le 2/3$, and $2/3 < p \le 1$, respectively.

\end{proof}
 
\begin{theorem}
Let $X, Y \in \mathbb{R}^n$ be independent sample points from two independent distributions with finite means and variances. Then we have the following:
\begin{itemize}
\item[(a)] $d_{RDPC}(X,Y|0, p,\vec{w})$ converges to one almost surely as $n \rightarrow \infty$.
\item[(b)] $d_{RDPC}(X,Y|1, 1,\vec{w})$ converges to $\E|X_1-Y_1|$ almost surely as $n \rightarrow \infty$.
\end{itemize}

\end{theorem}

\begin{proof}
For (a), we have
\beas
d_{RDPC}(X,Y|0, p,\vec{w})= d_P(X,Y)  = 1-\Corr(X,Y) = 1-\frac{\sum_{j=1}^n X_jY_j/n-\bar{X}\bar{Y}}{ S_X S_Y} .
\enas
By the strong law of large numbers, the last expression converges almost surely to 
\beas
1 - \frac{\E X_{1}Y_{1} - \E X_{1} \E Y_{1}}{\sqrt{\Var(X_1)\Var(Y_1)}} = 1- \frac{\Cov(X_1,Y_1)}{\sqrt{\Var(X_1)\Var(Y_1)}} = 1,
\enas
as $n \rightarrow \infty$. The last equality holds since $X$ and $Y$ are independent.

For (b), it can be seen that
\beas
d_{RDPC}(X,Y|1, 1,\vec{w}) = \text{RankDiff}(X,Y|1,\vec{w}) = \sum_{j=1}^n \frac{R_{(n-j+1)}}{n} = \sum_{j=1}^n \frac{R_{j}}{n} = \sum_{j=1}^n \frac{|X_j-Y_j|}{n}.
\enas
By the strong law of large numbers, the last expression converges almost surely to $
\E|X_1-Y_1|$ as $n \rightarrow \infty$. 

\end{proof}

\begin{figure*}[h]
\centering
\resizebox{15cm}{!}{%
\begin{tabular}{ccc}
\hline\hline
\\
\centering\includegraphics[width=2cm]{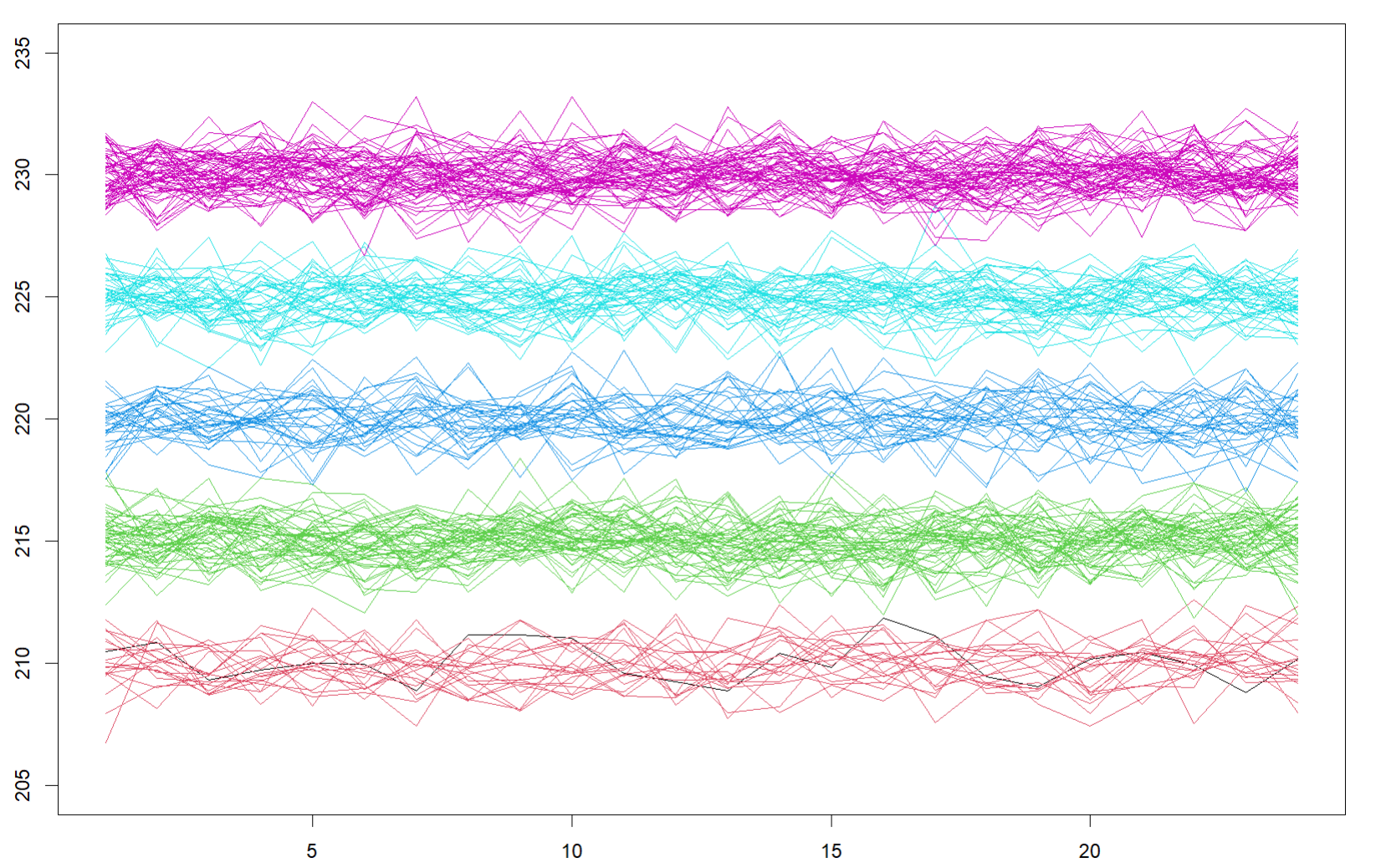} & \includegraphics[width=2cm]{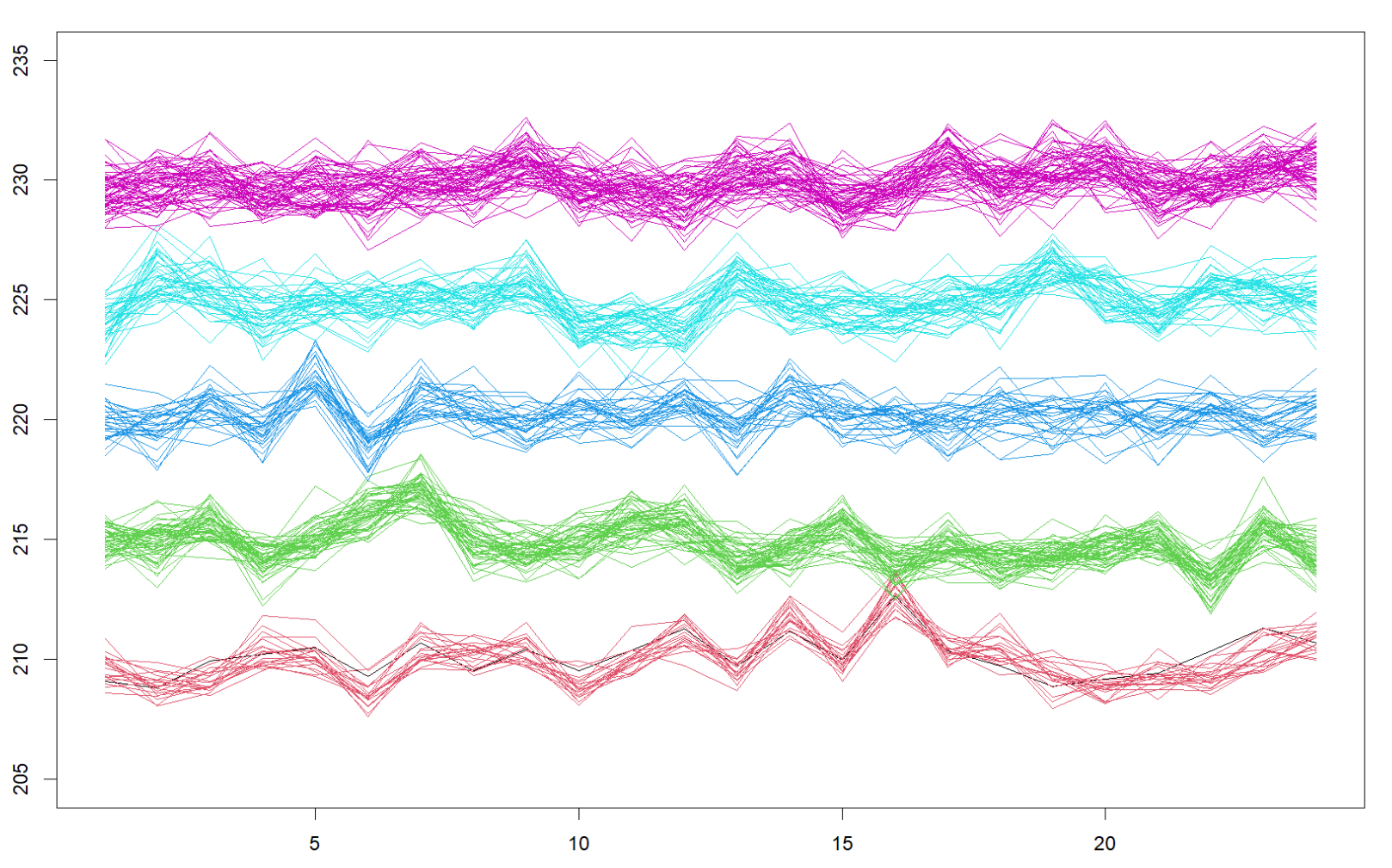} & \includegraphics[width=2.6cm]{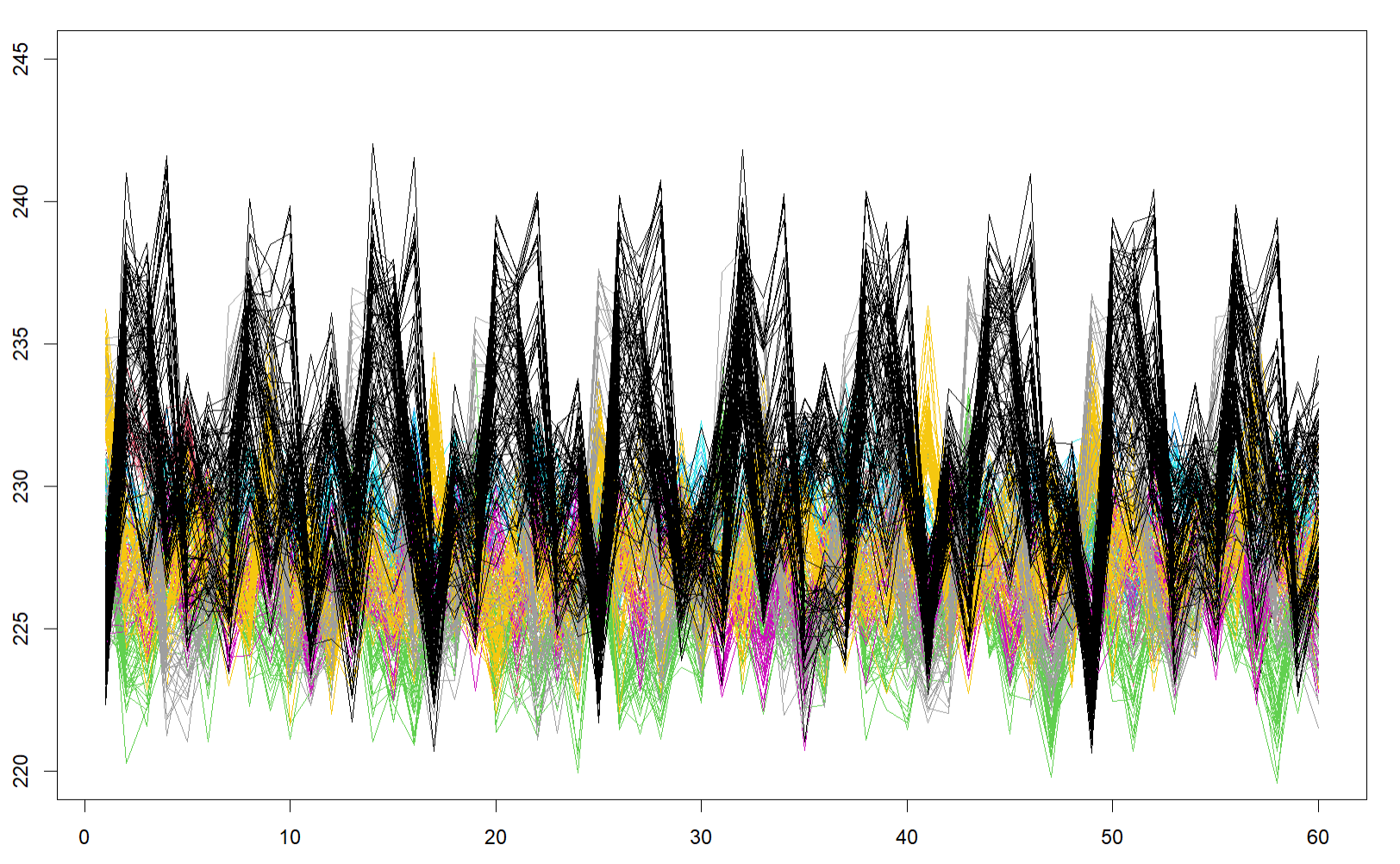} \\ \tiny D1&\tiny D2&\tiny C1 \\
\tiny\text{ n $=$ 200, \hspace{0.1cm} K $=$ 5, \hspace{0.1cm} T $=$ 24}&\tiny\text{ n $=$ 200, \hspace{0.1cm} K $=$ 5, \hspace{0.1cm} T $=$ 24}&\tiny\text{ n $=$ 310, \hspace{0.1cm} K $=$ 8, \hspace{0.1cm} T $=$ 60} \\
\centering\includegraphics[width=2.1cm]{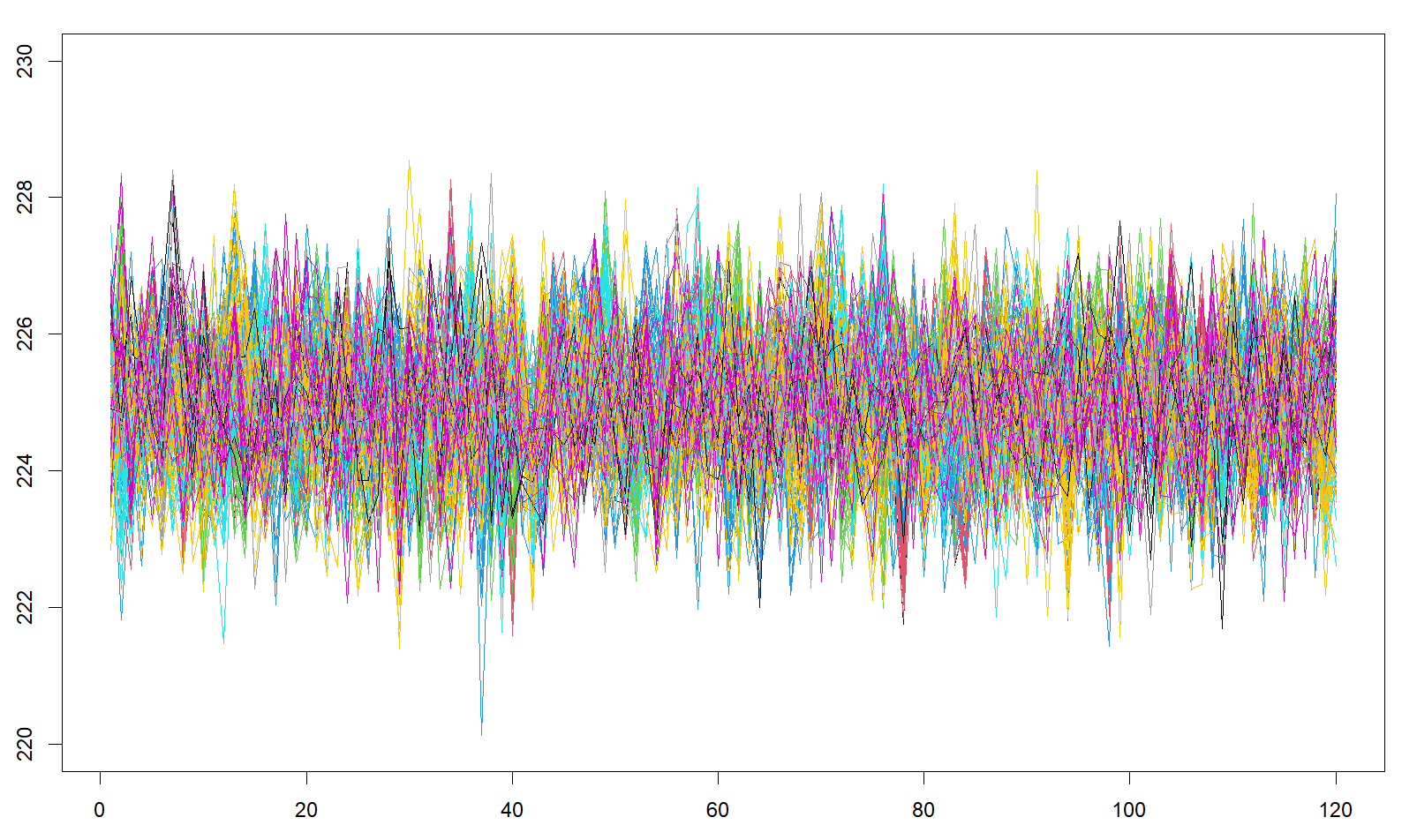} & \includegraphics[width=2cm]{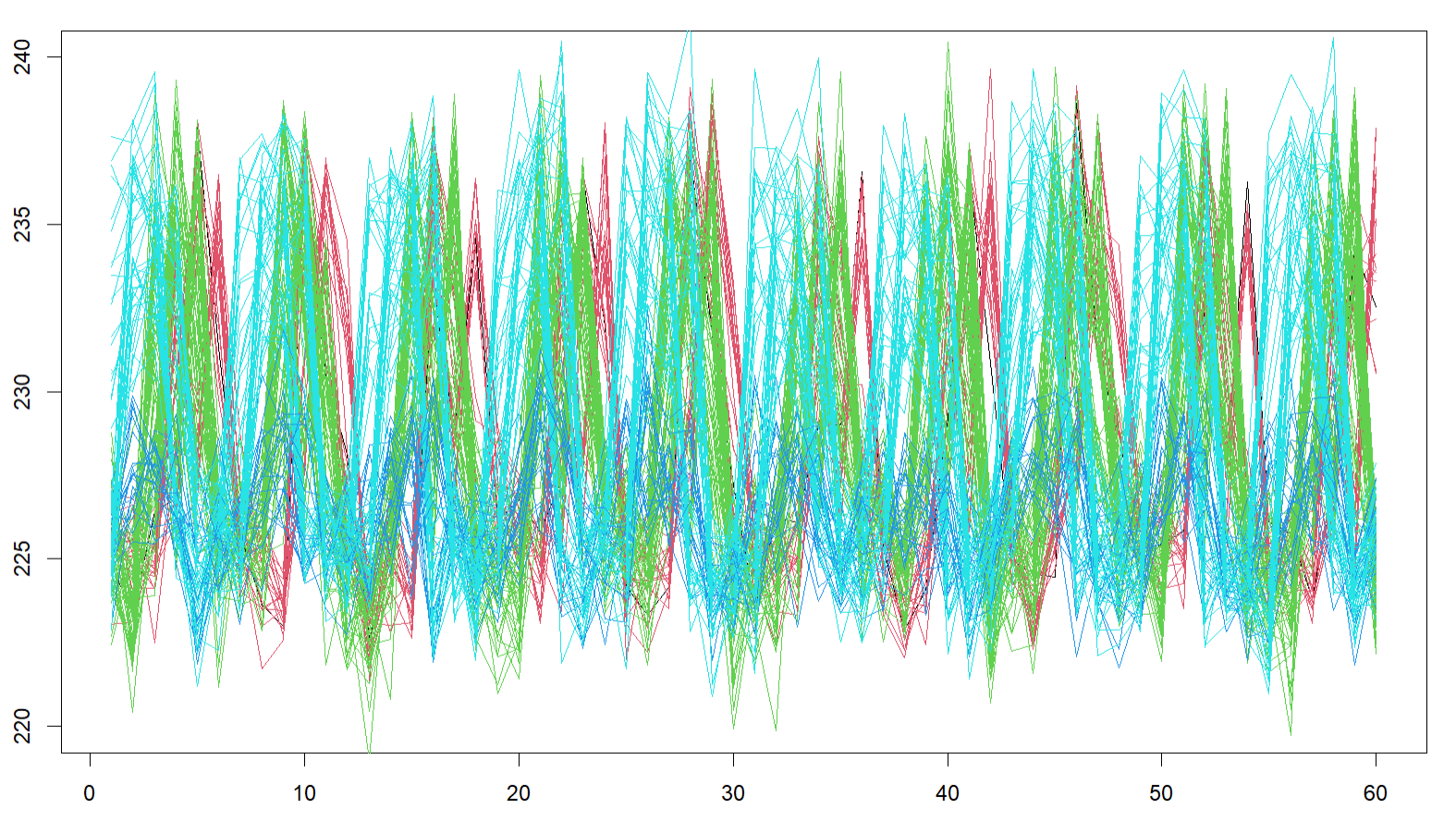} & \includegraphics[width=2cm]{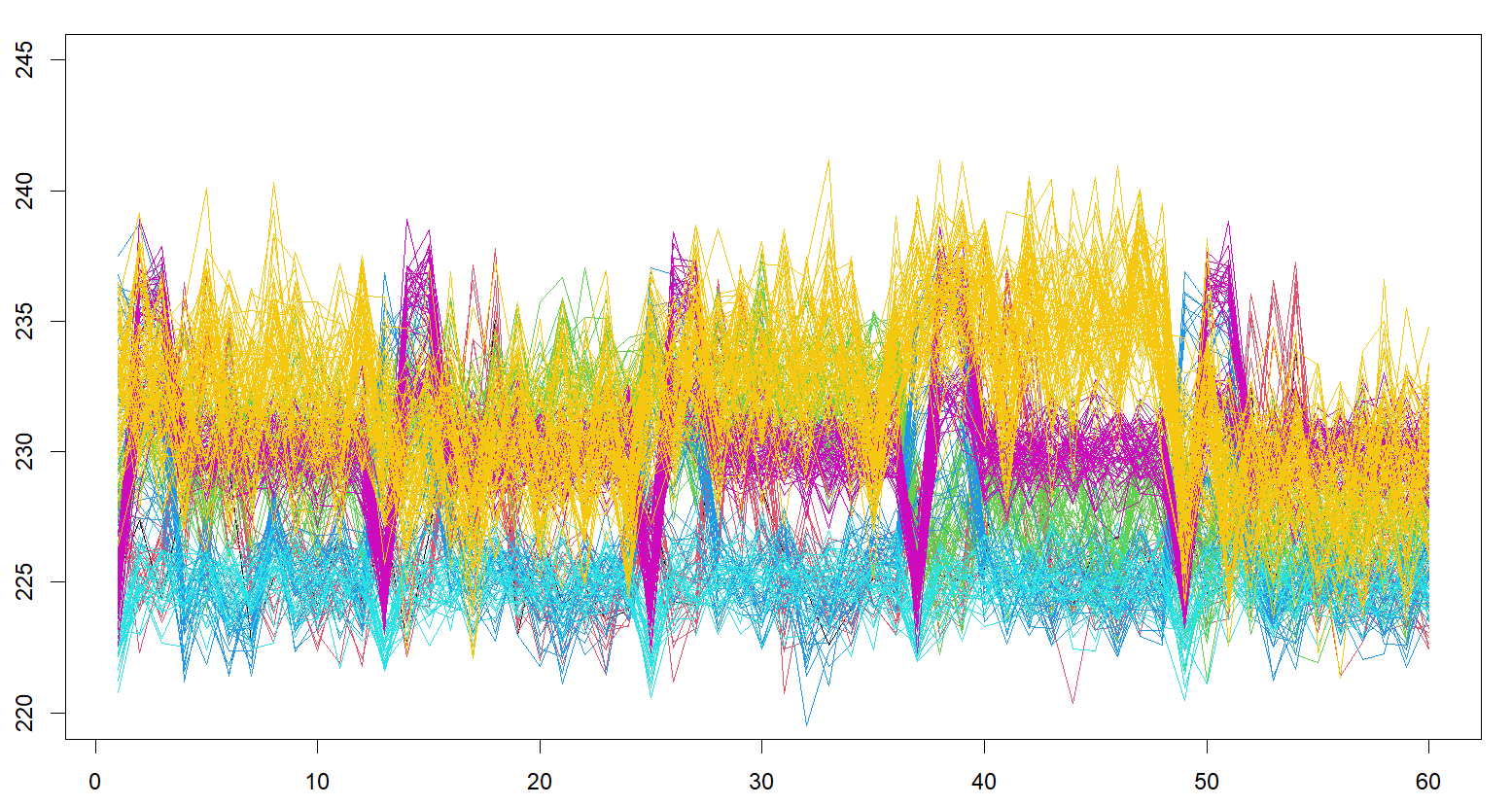} \\\tiny C2&\tiny M1&\tiny M2\\
\tiny\text{ n $=$ 150, \hspace{0.1cm} K $=$ 6, \hspace{0.1cm} T $=$ 120}&\tiny\text{ n $=$ 130, \hspace{0.1cm} K $=$ 4, \hspace{0.1cm} T $=$ 60}&\tiny\text{ n $=$ 270, \hspace{0.1cm} K $=$ 6, \hspace{0.1cm} T $=$ 60} \\
\centering\includegraphics[width=2cm]{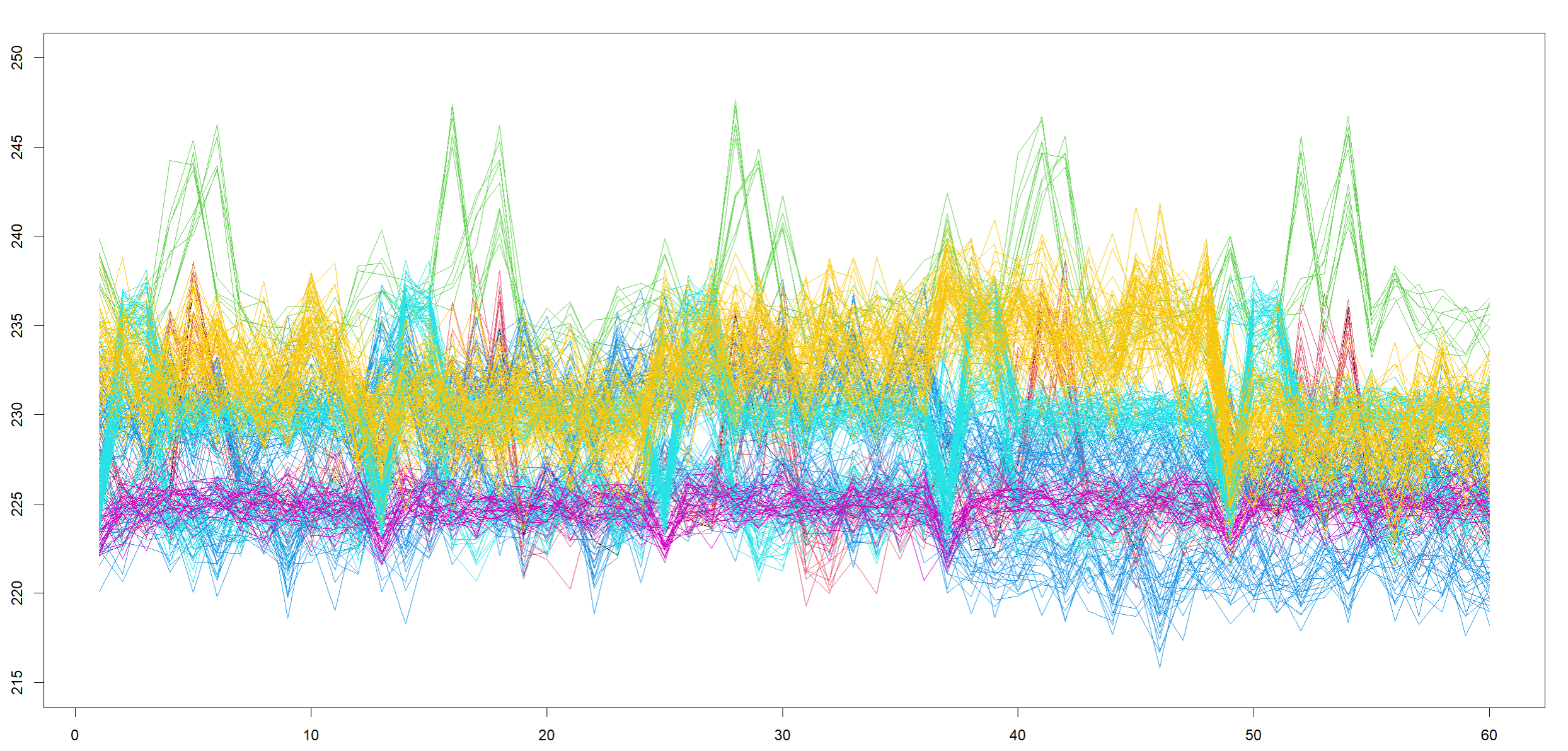} & \includegraphics[width=2cm]{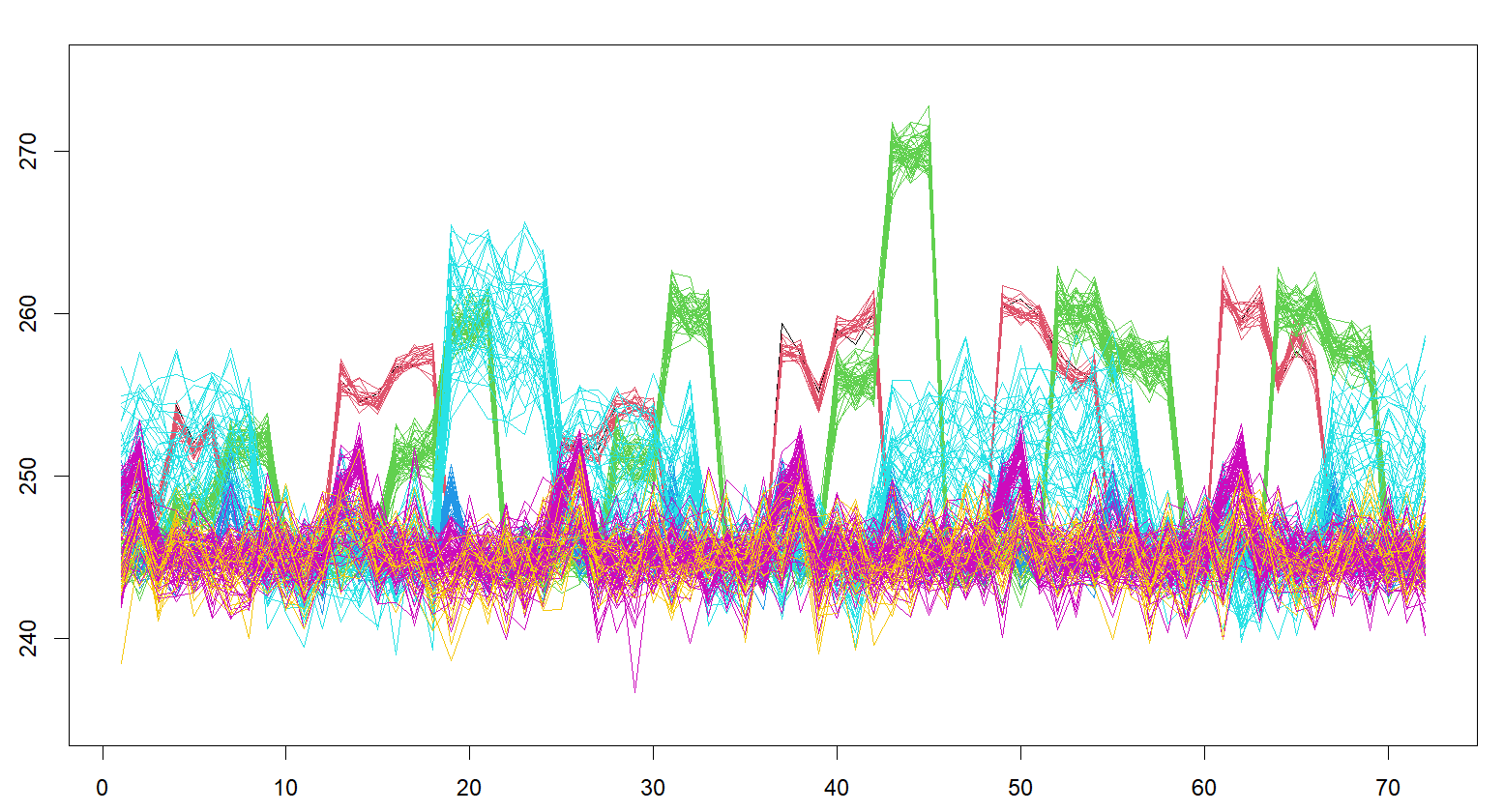} & \includegraphics[width=2.7cm]{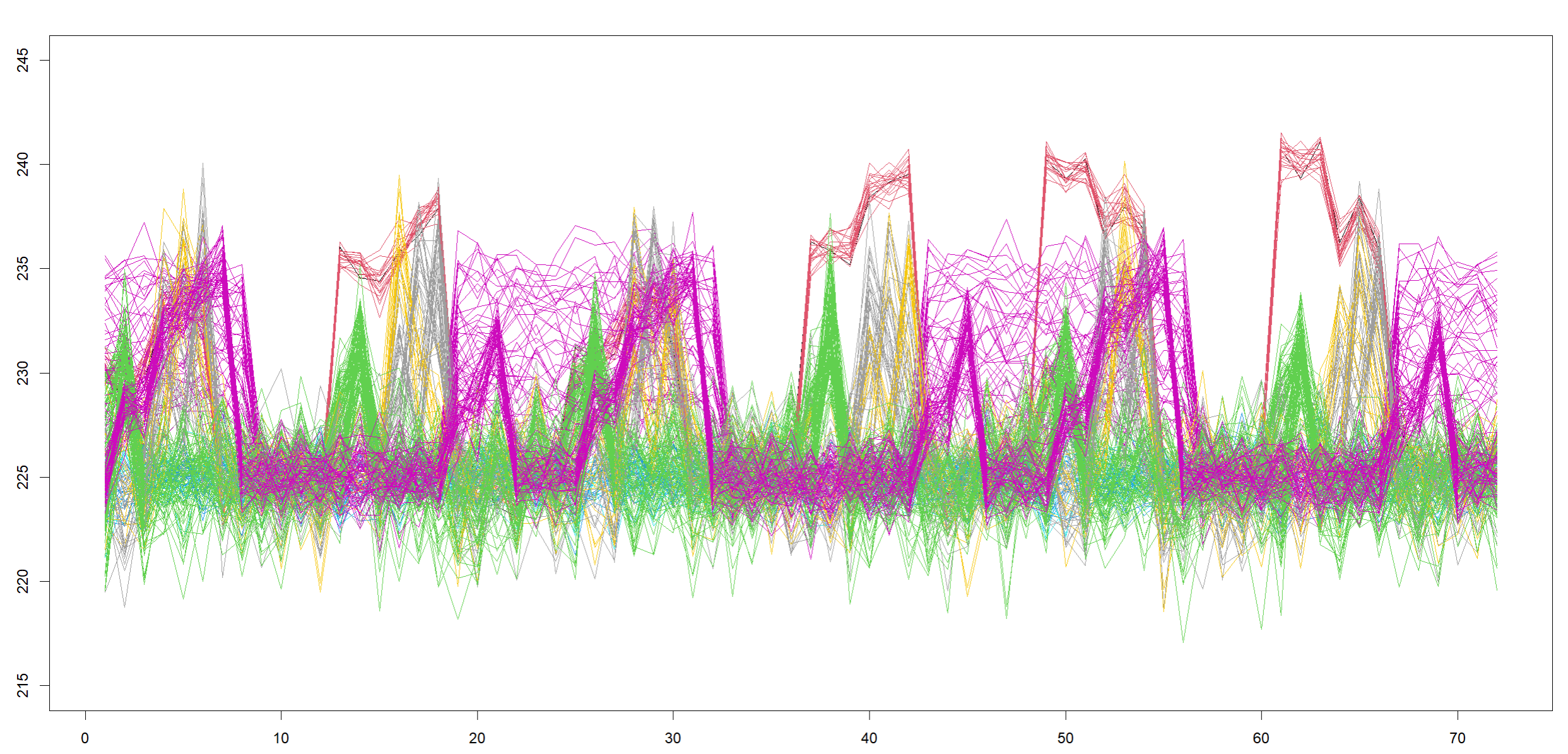}
\\\tiny M3&\tiny MC1&\tiny MC2\\
\tiny\text{ n $=$ 310, \hspace{0.1cm} K $=$ 6, \hspace{0.1cm} T $=$ 60}&\tiny\text{ n $=$ 240, \hspace{0.1cm} K $=$ 6, \hspace{0.1cm} T $=$ 72}&\tiny\text{ n $=$ 260, \hspace{0.1cm} K $=$ 7, \hspace{0.1cm} T $=$ 72} \\ \\

\hline \hline

\end{tabular}
}

\vspace{0.3cm}

\parbox{0.9\textwidth}{Note: n, K, T are the number of time series data points, the actual number of clusters, and the number of time steps, respectively.}
\caption{Artificial datasets}
\label{ArtfData}
\end{figure*}

\section{Experimental results}\label{sec:exp} 
In this section, we evaluate the performance of hierarchical clustering using our introduced dissimilarity measure on simulated datasets, which are illustrated in Figure~\ref{ArtfData} and can be summarized as follows:
\begin{list}{}{} \label{datagroup}
\item{\bf Group D: }{Each cluster in each dataset is generated from an independent Gaussian distribution with a different mean.}

\item{\bf Group C: }{Each cluster is generated from an independent Gaussian distribution with the same mean. Members of each cluster are highly correlated, but members from different clusters are slightly correlated or uncorrelated.}
\item{\bf Group M: }{This group is generated similarly to the first two groups, but each cluster in each dataset exhibits trends, seasonal patterns, and peaks at different times.}
\item{\bf Group MC: }{This group is similar to RCM but more complicated.}
\end{list}

For groups M and MC, we expect our proposed algorithm to be more suitable than the well-known ones. As discussed in the introduction, we sometimes face real-life situations similar to groups M and MC, such as water and electricity consumption and medical laboratory results. It is thus important to be able to differentiate the patterns observed in different groups within such datasets. 

The experiments are described in the following three subsections. The first presents a sensitivity analysis of the parameters of the proposed method. The second compares our new clustering scheme with well-known alternatives, including correlation-based hierarchical clustering, DTW, GAK, and basic K-means clustering, on the basis of their accuracy when setting the correct number of groups. The last explores the possibility of detecting the final number of clusters by the classic elbow method.


\subsection{Sensitivity analysis}\label{subsec:sensitivity}

\begin{table*}[h]
\centering
\resizebox{16cm}{!}{%
\begin{tabular}{ccc}

\hline \hline

\multicolumn{2}{c}{\multirow{1}{*}{}} \\

\centering\includegraphics[width=5cm]{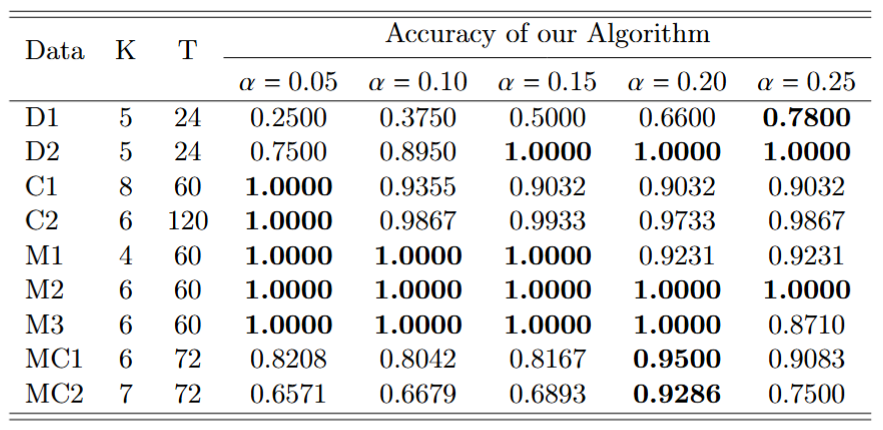} & \includegraphics[width=5.05cm]{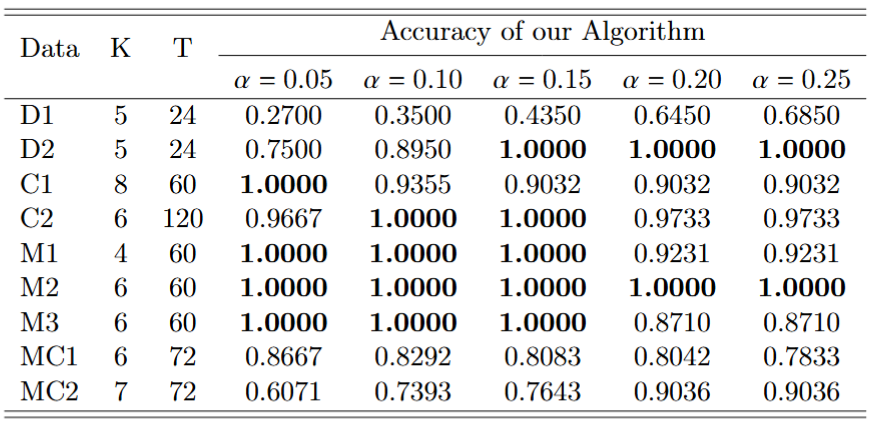}  \\ \tiny
 p = 0.1 &\tiny
 p = 0.2  \\

\centering\includegraphics[width=5cm]{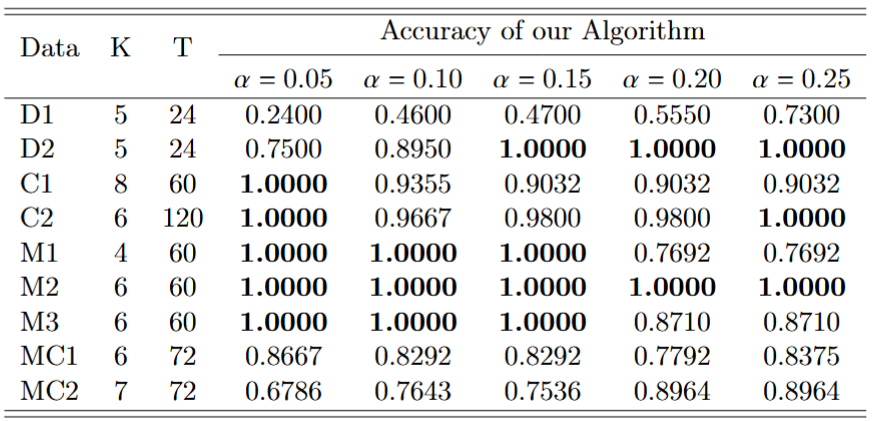} & \includegraphics[width=5.05cm]{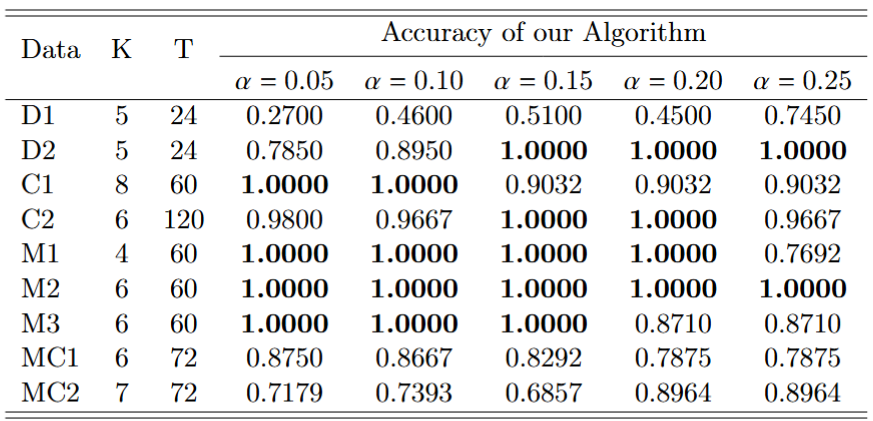}  \\ \tiny
 p = 0.3 &\tiny
 p = 0.5  \\

\centering\includegraphics[width=5cm]{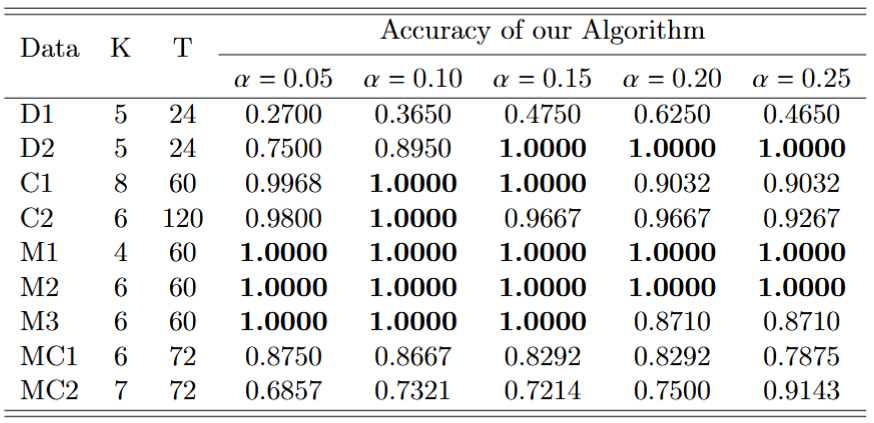} & \includegraphics[width=5.05cm]{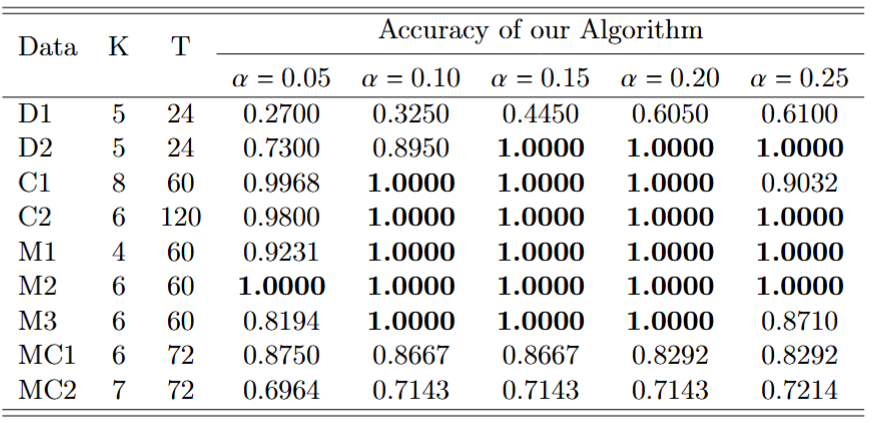}  \\ \tiny
 p = 0.7 &\tiny
 p = 1.0  \\

\hline \hline

\end{tabular}
}
\caption{Sensitivity analysis of the RDPC with respect to the parameters $p$ and $\alpha$. The highest accuracy of each dataset are bold.}
\label{sensitivity}
\end{table*}

In this subsection, we check how sensitive our proposed method is to changes in the parameter values by setting the parameters appearing in Definition~\ref{RDPCdef} to $\alpha \in \{0.05,0.10,0.15,0.20,0.25 \}$ and $p \in \{0.1,0.2,0.3,0.5,0.7,1\}$. We intentionally select $\alpha$ around $0.1$ to balance the importance of RankDiff and the correlation terms. We vary $p$ to demonstrate how it affects the results, although it is expected that a small $p$ will yield a better result. Note that when $p=1$, the RankDiff term is equivalent to the Manhattan distance. The sensitivity is evaluated on the basis of the accuracy of the proposed clustering algorithm when setting the correct number of groups for each simulated dataset. 

The results are presented in Table~\ref{sensitivity}. Since our main goal is to show that the proposed method is better suited to clustering the datasets in groups M and MC than the existing methods, we focus on the results of the last five rows of each sub-table. The algorithm is more sensitive to the parameter $\alpha$ than $p$, though both have a moderate effect on the accuracy within the range of 0.05 to 0.25. The combination of $p=0.1$ and $\alpha=0.2$ provides the best accuracy for groups M and MC; thus, we use these parameter values throughout the remainder of the work.

\begin{table*}[h]

\centering
\resizebox{11cm}{!}{%
\begin{tabular}{lccccccccccccc}
\hline\hline
\multirow{2}{*}{Data} & \multirow{2}{*}{K} & \multirow{2}{*}{T} & \multicolumn{5}{c}{Accuracy for correct K} \\ 
\cmidrule{4-8}
  & & & $d_{RDPC}$ & Correlation & DTW & GAK & Kmeans \\
\hline
D1 & 5 & 24 & 0.6600 & 0.2300 & {\bf 1.0000} & {\bf 1.0000} & {\bf 1.0000} \\
D2 & 5 & 24 & {\bf 1.0000} & 0.6600 & {\bf 1.0000} & {\bf 1.0000} & {\bf 1.0000} \\
C1 & 8 & 60 & 0.9032 & {\bf 0.9968} & 0.3548 & 0.7419 & 0.7742 \\
C2 & 6 & 120 & 0.9733 & {\bf 1.0000} & 0.1800 & 0.2133 & {\bf 1.0000} \\
M1 & 4 & 60 & {\bf 0.9231} & {\bf 0.9231} & 0.7000 & 0.5462 & 0.8462 \\
M2 & 6 & 60 & {\bf 1.0000} & 0.7481 & 0.4333 & 0.9630 & 0.9630 \\
M3 & 6 & 60 & {\bf 1.0000} & 0.8226 & 0.2419 & 0.5161 & 0.7419 \\
MC1 & 6 & 72 & {\bf 0.9500} & 0.8208 & 0.5542 & 0.5833 & 0.8000 \\
MC2 & 7 & 72 & {\bf 0.9286} & 0.6179 & 0.3643 & 0.4679 & 0.6464 \\
\hline\hline
\end{tabular}
}
\caption{Accuracy comparison between clustering algorithms on artificial datasets}
\label{Acctab}
\end{table*}

\subsection{Performance comparison}\label{subsec:performance comparison}

In this subsection, we compare the performance of our proposed method with the four existing approaches mentioned earlier. Table~\ref{Acctab} lists the accuracy of each algorithm when applied to each labeled dataset, which should indicate its suitability for that particular dataset.

Since the datasets in groups D and C are well separated in Euclidean distance and correlation, respectively, K-means and correlation-based hierarchical clustering are the best algorithms for these groups, as expected. DTW and GAK also work perfectly in group D since they are both based on the Euclidean distance. Our proposed $d_{RDPC}$ performs moderately well in both D and C but is not the best as it is a mixture of the two measures.

In contrast, our $d_{RDPC}$ method is substantially more accurate than the others for groups M and MC, with accuracy values of 0.9231, 1, 1, 0.95, and 0.9286. The K-means and correlation-based hierarchical clustering methods show moderate accuracy in both groups M and MC. DTW and GAK seem to be inappropriate for classifying these types of datasets since the time-warping concept should not be incorporated in cases with seasonal patterns. For instance, January should not be warped to match with April.


\subsection{Detecting the correct number of clusters}
In the previous subsection, we evaluated the performance of our proposed RDPC dissimilarity by comparing the resulting accuracy with other well-known methods when selecting the correct number of clusters. In reality, the correct number of clusters is unknown; hence, it is necessary to check that we can detect it. In this subsection, we verify that the classic elbow method, in which the $y$-axis is the sum of the applied within-cluster distances, can lead us to the correct number of clusters in most cases. \rcolor{A deeper concept of detecting a final number of clusters called cluster validity indices is not considered in this work since some serious effort needs to be done to make it fit to our new dissimilarity (see \cite{NW2024} and \cite{PW2025} for more details).}

\begin{table*}[h!]
\centering
\resizebox{13cm}{!}{%
\begin{tabular}{lccccccccccccccccccccccc}
\hline\hline
\multirow{3}{*}{Data} & \multirow{3}{*}{K} & &  \multicolumn{19}{c}{Elbow points} \\ 
\cmidrule{4-22}
  & & & \multicolumn{3}{c}{$d_{RDPC}$} & & \multicolumn{3}{c}{Correlation} & & \multicolumn{3}{c}{DTW} & & \multicolumn{3}{c}{GAK} & & \multicolumn{3}{c}{Kmeans} \\
  \cmidrule{4-6} \cmidrule{8-10} \cmidrule{12-14} \cmidrule{16-18} \cmidrule{20-22}
  & & & $e_{1}$ & $e_{2}$ & $e_{3}$ & & $e_{1}$ & $e_{2}$ & $e_{3}$ & & $e_{1}$ & $e_{2}$ & $e_{3}$ & & $e_{1}$ & $e_{2}$ & $e_{3}$ & & $e_{1}$ & $e_{2}$ & $e_{3}$ \\
\hline
D1 & 5 & & 2 & 4 & 8 & & 4 & 7 & 10 & & 2 & \textbf{5} & 7 & & \textbf{5} & 7 & - & & 2 & 3 & \textbf{5} \\
D2 & 5 & & 2 & 4 & 6 & & 3 & 4 & 12 & & 3 & \textbf{5}& 7 & & \textbf{5} & 7 & - & & 2 & 3 & \textbf{5} \\
C1 & 8 & & 3 & 5 & 7 & & 2 & 5 & \textbf{8} & & 2 & 4 & \textbf{8} & & 3 & 7 & \textbf{8} & & 2 & 5 & 7 \\
C2 & 6 & & 2 & 4 & \textbf{6} & & \textbf{6} & - & - & & 3 & - & - & & 3 & - & - & & \textbf{6} & - & - \\
M1 & 4 & & 2 & \textbf{4} & 6 & & 2 & 7 & 12 & & \textbf{4} & 9 & - & & 3 & - & - & & 2 & 5 & 8 \\
M2 & 6 & & 2 & 4 & \textbf{6} & & 5 & 7 & - & & 2 & 5 & - & & 8 & - & - & & 2 & \textbf{6} & - \\
M3 & 6 & & 3 & 5 & \textbf{6} & & 4 & 7 & 3 & & 4 & 10 & - & & 5 & 7 & - & & 2 & 7 & - \\
MC1 & 6 & & 2 & 4 & \textbf{6} & & 5 & 7 & 12 & & 4 & 7 & - & & 4 & 8 & - & & 4 & 7 & - \\
MC2 & 7 & & 3 & 4 & \textbf{7} & & 4 & 10 & 15 & & 3 & 6 & - & & 9 & - & - & & 5 & - & - \\
\hline\hline
\end{tabular}
}

\vspace{0.3cm}

\parbox{0.8\textwidth}{Note : The first three elbow points; $e_1$, $e_2$ and $e_3$ from applying each clustering algorithm to each dataset. The correct numbers of clusters are bold.}
\caption{Results of elbow method}
\label{Elbowtab}
\end{table*}

Table~\ref{Elbowtab} lists up to three elbow points obtained when applying each clustering algorithm to each dataset. For group D, the correct number of clusters for each dataset is within the first three elbow points for the DTW, GAK, and K-means methods. This is also the case for group C with the correlation-based hierarchical clustering method. 

On the other hand, for groups M and MC, there are no algorithms whose elbow points align with the true labels of the dataset, except for our proposed $d_{RDPC}$ method. This finding supports the argument that our algorithm is more suitable for datasets with greater diversity and complexity.

\section{Application to electricity consumption data}\label{sec:app}

It was shown in the previous section that our proposed algorithm outperforms traditional algorithms when applied to complicated datasets with temporal trends, seasonal patterns, and different peaks. Thus, in this section, we apply it to the electricity consumption dataset collected by the Provincial Electricity Authority (PEA) of Thailand. The dataset is a random sample\footnote{The PEA data are confidential, and we are allowed to access only a random sample of 1,200 users.} of records of monthly electricity consumption in kilowatt-hours (kWh) of 1,200 users over 36 months from January 2021 to December 2023. After removing all users with zero usage, there are 1,174 users remaining.

We first observe that there is a minority of users with unusually high usage. These should be excluded before applying a clustering algorithm; otherwise, they will be placed in their own small clusters. Hence, we remove these users by sequentially applying our proposed clustering method with two groups at a time and keeping only the lower usage group. This results in 1,017 regular users and 157 high-usage users. Descriptive statistics of the 157 users and a plot of their time series are shown in Table~\ref{Characoutlier} and Figure~\ref{outlier}, respectively. We now use our proposed algorithm to cluster these 1,017 regular users.

\begin{table*}[h!]
    \centering
    \begin{subtable}[t]{1\textwidth}
        \centering
        \resizebox{\textwidth}{!}{ 
        \begin{tabular}{ccccccccccccccccccccccccccc}
            \toprule
            \multirow{2}{*}{Cluster} & \multirow{2}{*}{N} & & \multicolumn{3}{c}{Yearly total} & & \multicolumn{3}{c}{Max} & & \multicolumn{3}{c}{Min} & & \multicolumn{3}{c}{Mean} & & \multicolumn{3}{c}{SD} \\
            \cmidrule{4-6} \cmidrule{8-10} \cmidrule{12-14} \cmidrule{16-18} \cmidrule{20-22}
            & & & 1 & 2 & 3 & & 1 & 2 & 3 & & 1 & 2 & 3 & & 1 & 2 & 3 & & 1 & 2 & 3 \\ 
            \hline \\
            0 & 157 & & 9052.48 & 8933.10 & 9591.73 & & 858.06 & 814.13 & 965.13 & & 596.20 & 669.55 & 640.86 & & 754.37 & 744.42 & 799.31 & & 91.87 & 54.08 & 102.91\\
            \bottomrule
        \end{tabular}
        }
    \end{subtable}
    
    \hspace{2\textwidth} 
    
    \begin{subtable}[t]{1\textwidth}
        \centering
        \resizebox{0.4\textwidth}{!}{ 
        \begin{tabular}{ccccccccccccccccccccc}
            \toprule
            \multirow{2}{*}{Cluster} & & \multirow{2}{*}{Mean Cor} & & \multirow{2}{*}{SD Cor} & & \multicolumn{3}{c}{Max month} & & \multicolumn{3}{c}{Min month} \\
            \cmidrule{7-9} \cmidrule{11-13}
            & & & & & & 1 & 2 & 3 & & 1 & 2 & 3  
             \\ 
            \hline \\
            0 & & 0.20 & & 0.35 & & 5 & 5 & 5 & & 2 & 1 & 1  \\
            \bottomrule
        \end{tabular}
        }
    \end{subtable}
\caption{Extremely high usage users} 
\label{Characoutlier}
\end{table*}

\begin{figure}[H]
\centering\includegraphics[width=12cm]{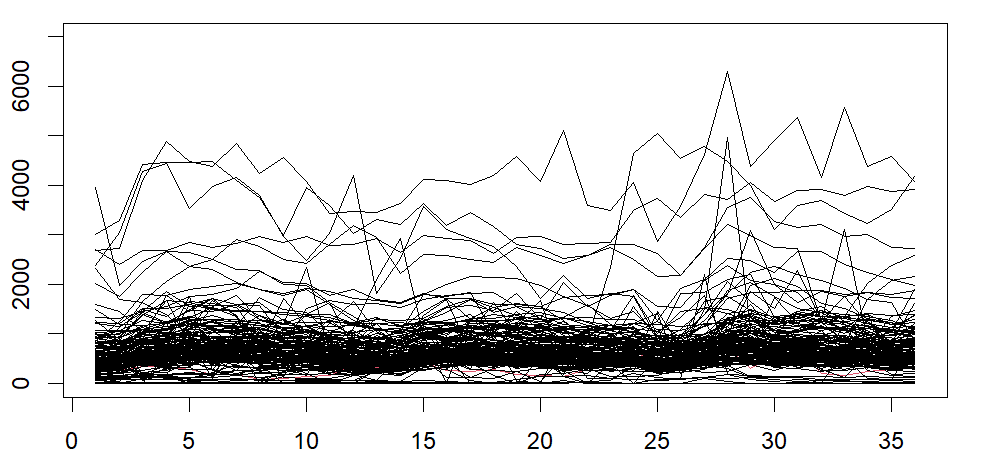}
\caption{High usage users}
\label{outlier}
\end{figure}

We determine the optimal number of clusters using the elbow method, as illustrated in Figure~\ref{elbowel}. Four distinct elbow points are observed with another slight elbow at nine, leading us to select seven as the final number of clusters. This choice ensures a diverse grouping structure, allowing for a thorough examination of each group's characteristics, and will be used in the subsequent steps of the clustering process. 

\begin{figure}[H]
\centering\includegraphics[width=12cm]{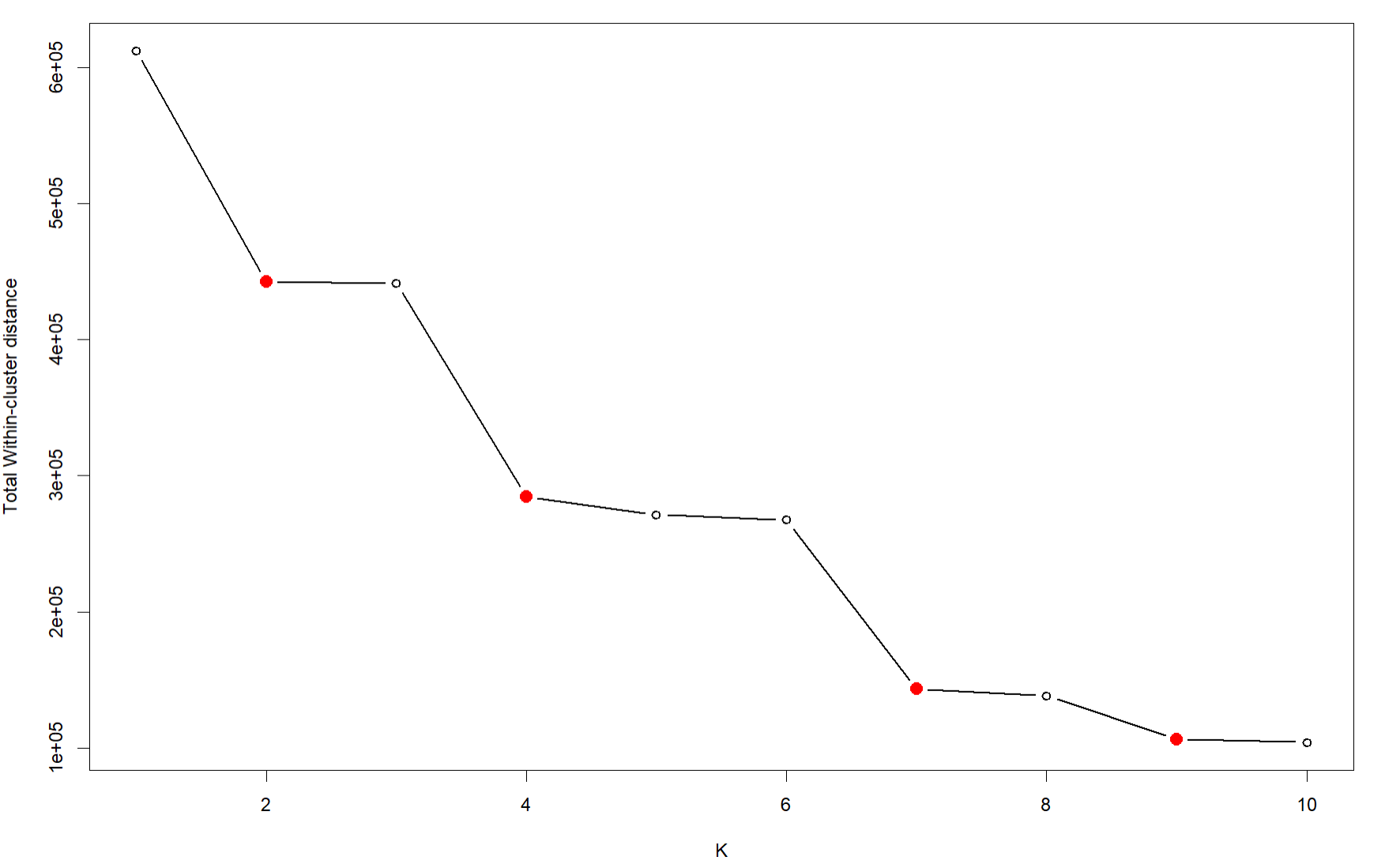}
\caption{Elbow method on 1,017 remaining users}
\label{elbowel}
\end{figure}

\begin{figure}[H]
\centering\includegraphics[width=16cm]{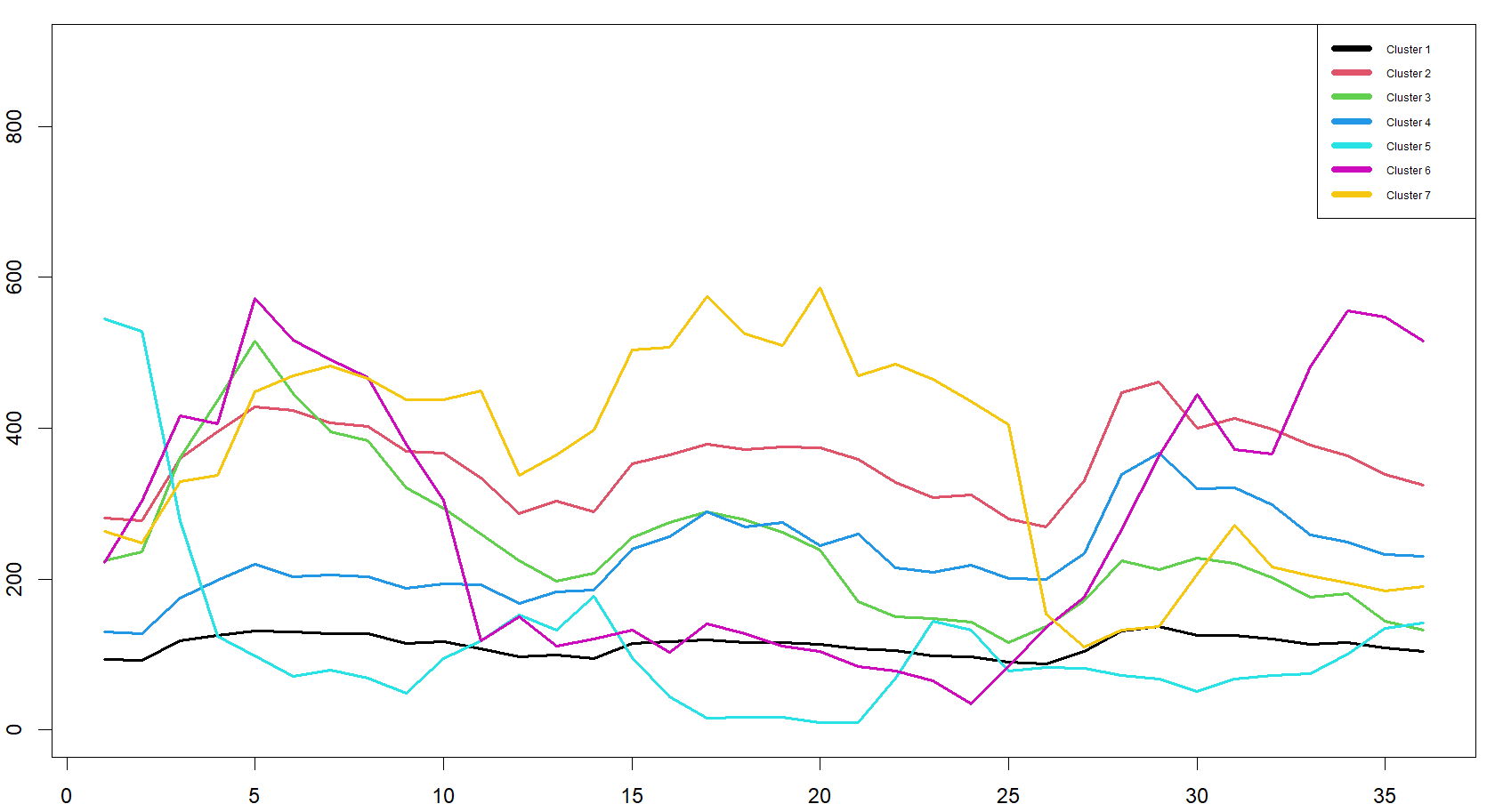}
\caption{Average electricity consumption in kWh of each cluster from January 2021 to December 2023.}
\label{volumeelec}
\end{figure}




\begin{table*}[h!]
    \centering
    \begin{subtable}[t]{1\textwidth}
        \centering
        \resizebox{1.0\textwidth}{!}{ 
        \begin{tabular}{ccccccccccccccccccccccccccc}
            \toprule 
            \multirow{2}{*}{Cluster} & \multirow{2}{*}{N} & & \multicolumn{3}{c}{Yearly total} & & \multicolumn{3}{c}{Max} & & \multicolumn{3}{c}{Min} & & \multicolumn{3}{c}{Mean} & & \multicolumn{3}{c}{SD} \\
            \cmidrule{4-6} \cmidrule{8-10} \cmidrule{12-14} \cmidrule{16-18} \cmidrule{20-22}
            & & & 1 & 2 & 3 & & 1 & 2 & 3 & & 1 & 2 & 3 & & 1 & 2 & 3 & & 1 & 2 & 3 \\ 
            \hline \\
            1 & 863 & & 1381.44 & 1294.00 & 1360.86 & & 131.23 & 118.70 & 137.30 & & 92.09 & 94.00 & 87.66 & & 115.12 & 107.83 & 113.40 & & 14.37 & 9.04 & 15.58\\
            2 & 68 & & 4332.06 & 4115.43 & 4404.69 & & 428.07 & 378.93 & 461.34 & & 277.49 & 289.16 & 269.53 & & 361.00 & 342.95 & 367.06 & & 55.02 & 32.83 & 60.84\\
            3 & 18 & & 4095.50 & 2611.89 & 2142.94 & & 515.00 & 289.44 & 227.56 & & 223.94 & 142.50 & 115.56 & & 341.29 & 217.66 & 178.58 & & 96.95 & 55.50 & 39.24 \\
            4 & 55 & & 2201.96 & 2843.18 & 3251.00 & & 219.29 & 289.33 & 367.20 & & 127.95 & 183.11 & 199.75 & & 183.50 & 236.93 & 270.92 & & 28.92 & 34.77 & 56.25 \\
            5 & 5 & & 2204.60 & 862.80 & 1023.60 & & 544.80 & 176.80 & 142.00 & & 48.40 & 9.60 & 51.00 & & 183.72 & 71.90 & 85.30 & & 175.07 & 61.73 & 27.40 \\
            6 & 4 & & 4345.25 & 12911.25 & 4307.25 & & 572.50 & 140.75 & 555.75 & & 118.25 & 33.75 & 83.50 & & 362.10 & 100.94 & 358.94 & & 145.03 & 30.76 & 162.19 \\
            7 & 4 & & 4706.50 & 5824.50 & 2405.50 & & 482.50 & 586.00 & 405.25 & & 247.75 & 365.00 & 110.25 & & 392.21 & 485.38 & 200.46 & & 84.01 & 64.97 & 77.83 \\
            \bottomrule
        \end{tabular}
        }
    \end{subtable}
    
    \hspace{2\textwidth} 
    
    \begin{subtable}[t]{1\textwidth}
        \centering
        \resizebox{0.6\textwidth}{!}{ 
        \begin{tabular}{ccccccccccccccccccccccccccccccc}
            \toprule
            \multirow{2}{*}{Cluster} & & \multirow{2}{*}{Mean Cor} & & \multirow{2}{*}{SD Cor} & & \multicolumn{3}{c}{Max month} & & \multicolumn{3}{c}{Min month} &  & \multicolumn{2}{c}{Trend} \\
            \cmidrule{7-9} \cmidrule{11-13} \cmidrule{15-16}
            & & & & & & 1 & 2 & 3 & &  1 & 2 & 3 & & 1 & 2 \\
            \hline \\
            1 & & 0.12 & & 0.31 & & 5 & 5 & 5 & & 2 & 2 & 2 & & \begin{tikzpicture}
            \draw[thick] (0.5,0.5) -- ++(0.5,-0.5)   -- ++(0.5,0.5);
            \end{tikzpicture} & \begin{tikzpicture}
            \draw[thick] (0.5,0) -- ++(0.5,0) -- ++(0.5,0); 
            \end{tikzpicture} \\
            2 & & 0.37 & & 0.28 & & 5 & 5 & 5 & & 2 & 2 & 2 & & \begin{tikzpicture}
             \draw[thick] (0.5,0) -- ++(0.5,0)   -- ++(0.5,0);
            \end{tikzpicture} & \begin{tikzpicture}
            \draw[thick] (0.5,0) -- ++(0.5,0) -- ++(0.5,0.5); 
            \end{tikzpicture} 
            \begin{tikzpicture}
            \draw[thick] (0.5,0.5) -- ++(0.5,-0.5) -- ++(0.5,0); 
            \end{tikzpicture}\\
            3 & & 0.51 & & 0.18 & & 5 & 5 & 6 & & 1 & 12 & 1 & & \begin{tikzpicture}
            \draw[thick] (0.5,0.5) -- ++(0.5,-0.5)   -- ++(0.5,-0.5);
            \end{tikzpicture} & \begin{tikzpicture}
            \draw[thick] (0.5,0.5) -- ++(0.5,-0.5) -- ++(0.5,0); 
            \end{tikzpicture} \\
            4 & & 0.29 & & 0.30 & & 5 & 5 & 5 & & 2 & 1 & 2 & & \begin{tikzpicture}
            \draw[thick] (0.5,-0.5) -- ++(0.5,0.5)   -- ++(0.5,0.5);
            \end{tikzpicture} & \begin{tikzpicture}
            \draw[thick] (0.5,-0.5) -- ++(0.5,0.5) -- ++(0.5,0); 
            \end{tikzpicture} \\
            5 & & 0.59 & & 0.18 & & 1 & 2 & 12 & & 9 & 8 & 6 & & \begin{tikzpicture}
            \draw[thick] (0.5,0.5) -- ++(0.5,-0.5)   -- ++(0.5,0.5);
            \end{tikzpicture} & \begin{tikzpicture}
            \draw[thick] (0.5,0.5) -- ++(0.5,-0.5) -- ++(0.5,-0.5); 
            \end{tikzpicture} \\
            6 & & 0.59 & & 0.12 & & 5 & 5 & 10 & & 11 & 12 & 1 & & \begin{tikzpicture}
            \draw[thick] (0.5,0.5) -- ++(0.5,-0.5)   -- ++(0.5,0.5);
            \end{tikzpicture} & - \\
            7 & & 0.50 & & 0.23 & & 7 & 8 & 1 & & 2 & 1 & 3 & & \begin{tikzpicture}
            \draw[thick] (0.5,-0.5) -- ++(0.5,0.5)   -- ++(0.5,-0.5);
            \end{tikzpicture} & \begin{tikzpicture}
            \draw[thick] (0.5,0) -- ++(0.5,0) -- ++(0.5,-0.5); 
            \end{tikzpicture} 
            \begin{tikzpicture}
            \draw[thick] (0.5,0.5) -- ++(0.5,-0.5) -- ++(0.5,-0.5); 
            \end{tikzpicture}\\
            \bottomrule
        \end{tabular}
        }
    \end{subtable}
\caption{Clusters characteristics} 
\label{Charac7}
\end{table*}

    

Table~\ref{Charac7} lists some characteristics of each of the seven groups, including the monthly mean, maximum, minimum, standard deviation, correlation mean, and correlation standard deviation, as well as the month with the highest usage. We also show the two strongest consumption trends in three years, with upward and downward trends identified when the usage increases or decreases, respectively, by more than 10\% from the previous year. Cluster~1 is the largest group with 863 users, which we consider the standard users. Since this dataset is only a small sample from the whole country, we claim that the very small clusters like Clusters 5, 6, and~7 represent classes of similar users across the whole country that are worth considering. Most of the clusters show the highest consumption during the summer and the lowest during the winter. This is expected in the context of Thailand, which has high temperatures during the summer and moderate temperatures during the winter. Figure~\ref{volumeelec} also shows the overall monthly trend of each cluster. They can be summarized as follows:
\begin{list}{}{} \label{characteristic}
\item{\bf Cluster 1 (Standard users): }{This is the group with the lowest electricity consumption, and within the group, there is no clear pattern of electricity usage in any particular direction. The average usage is about 110~kWh per month.}
\item{\bf Cluster 2 (High stable users): }{This is the group with high and stable consumption, with an average of about 350~kWh per month. The users in this group are moderately correlated, with a correlation of 0.37.}
\item{\bf Cluster 3 (Declining users): }{This is the group with highly correlated users, with a correlation of 0.51, whose usage significantly dropped in both years relative to the previous years (by 36\% and 18\% on average, respectively).}
\item{\bf Cluster 4 (Increasing users): }{This group is similar to Cluster~3 but with the opposite trend. The users in this group had a similar level of usage to Cluster~3 but with a lower correlation. The usages increased in both years relative to the previous years (by 29\% and 14\% on average, respectively).}
\item{\bf Cluster 5 (Winter users): }{This is the group with highly correlated users, with a correlation of 0.59, with the special characteristic that their highest usage was during the winter. This is unusual in Thailand.}
\item{\bf Cluster 6 (2022 dropped users): }{This is the group with highly correlated users, with a correlation of 0.59, with the special characteristic that the 2022 consumption dropped relative to the previous year (by 72\% on average). } 
\item{\bf Cluster 7 (2023 dropped users): }{This is the group with highly correlated users, with a correlation of 0.51, with the special characteristic that the 2023 consumption dropped relative to the previous year (by 59\% on average).} 
\end{list}

\section{Conclusion} \label{sec:conclusion}

In this work, we propose a new dissimilarity measure called RDPC. It is intended to be used in hierarchical clustering of time series data. The RDPC is defined as an interpolation between the Pearson correlation dissimilarity and a weighted average of a specified fraction of the largest element-wise differences between two time series. The main benefits of using RDPC for clustering time series are as follows: It can distinguish time series on the basis of both Euclidean distance and correlation, and it has parameters that can be adjusted by users, providing flexibility.

To demonstrate benefits, we evaluate the performance compared with hierarchical clustering using the Pearson correlation dissimilarity measure, DTW, GAK, and K-means clustering on nine simulated datasets. These include four complicated datasets with different seasonal patterns, peaks, correlations, and Euclidean distances. Our proposed RDPC outperforms the existing algorithms in these complicated cases.

Finally, we apply our new method to PEA electricity consumption data and cluster a random sample of users in Thailand into seven main groups, namely, standard, high stable, declining, increasing, winter, 2022 dropped, and 2023 dropped users. These interesting characteristics of each cluster follow from the way our RDPC is defined. 

Future research directions include applying our RDPC to other well-known clustering algorithms, adding more parameters to extend the RDPC, and implementing the RDPC in a wider variety of real-world situations.

\section*{Acknowledgements}
Both authors appreciate the Provincial Electricity Authority (PEA) that allows us to access the electricity consumption dataset. Chutiphan would like to thank Development and Promotion of Science and Technology Talents Project Master’s Degree Research Scholarship from
Institute for the Promotion of Teaching Science and Technology, and Nathakhun would like to also thank National Research Council of Thailand (NRCT), Grant number: N42A660991 (2023) for the project partial financial support.

\end{document}